\documentclass{article}
\usepackage[preprint,nonatbib]{neurips_2024}
\usepackage{graphicx}
\usepackage{booktabs}
\usepackage[misc]{ifsym}

\usepackage{mwe}

\usepackage[utf8]{inputenc} 
\usepackage{hyperref}       
\usepackage{url}            
\usepackage{nicefrac}       
\usepackage{microtype}      

\usepackage[dvipsnames]{xcolor}
\usepackage{subcaption}
\usepackage{amsmath}
\usepackage{amsthm}
\usepackage{amsfonts}
\usepackage{mathrsfs}
\usepackage{dsfont}
\usepackage{chemfig}
\usepackage{tikz,ifthen}
\usetikzlibrary{shapes}
\usetikzlibrary{plotmarks}
\usetikzlibrary{decorations.pathreplacing,shapes.misc,shapes.symbols}
\usetikzlibrary{arrows, "shapes.geometric"}
\usetikzlibrary{calc}
\usetikzlibrary{backgrounds}
\usetikzlibrary{shapes}
\usepackage{adjustbox}
\usetikzlibrary{positioning}

\usepackage{tikz-cd}                    
\usepackage{mathtools}
\usepackage{verbatim}
\usepackage{enumitem}
\usepackage{mathdots}
\usepackage{comment}
\usepackage{cleveref}
\usepackage{float}

\newtheorem{thm}{Theorem}
\newtheorem{prop}{Proposition}

\newcommand{\Rule}{\mathbf{R}}
\newcommand{\RuleFC}{\Rule_{\operatorname{FC}}}
\newcommand{\RuleCV}{\Rule_{\operatorname{CNN}}}
\newcommand{\RuleMol}{\Rule_{\operatorname{Mol}}}
\newcommand{\RuleFinal}{\Rule_{\operatorname{Aggr}}}
\newcommand{\Mod}[1]{\ (\mathrm{mod}\ #1)}
\newcommand{\weightset}{\Theta}
\newcommand{\ruleset}{\mathcal{R}}
\newcommand{\graphlabeling}{l}
\newcommand{\labelset}{\mathcal{L}}
\newcommand{\propertyfunction}{p}

\definecolor{tab_blue}{HTML}{1f77b4}
\definecolor{tab_orange}{HTML}{ff7f0e}
\definecolor{tab_green}{HTML}{2ca02c}
\definecolor{tab_red}{HTML}{d62728}
\definecolor{tab_purple}{HTML}{9467bd}
\definecolor{tab_brown}{HTML}{8c564b}
\definecolor{tab_pink}{HTML}{e377c2}
\definecolor{tab_gray}{HTML}{7f7f7f}
\definecolor{tab_olive}{HTML}{bcbd22}
\definecolor{tab_cyan}{HTML}{17becf}

\title{Rule Based Learning with Dynamic (Graph) Neural Networks}

\author{%
  Florian Seiffarth \\
  Department of Computer Science\\
  University of Bonn\\
  53113 Bonn, Germany \\
  \texttt{seiffarth@cs.uni-bonn.de} \\
  }

\begin{document}

\maketitle              

\begin{abstract}
  A common problem of classical neural network architectures is that additional information or expert knowledge cannot be naturally integrated into the learning process.
  To overcome this limitation, we propose a two-step approach consisting of (1) generating rule functions from knowledge and (2) using these rules to define  rule based layers -- a
  new type of dynamic neural network layer.
  The focus of this work is on the second step, i.e., rule based layers that are designed to dynamically arrange learnable parameters in the weight matrices and bias vectors depending on the input samples.
  Indeed, we prove that our approach generalizes classical feed-forward layers such as fully connected and convolutional layers by choosing appropriate rules.
  As a concrete application we present rule based graph neural networks (RuleGNNs) that overcome some limitations of ordinary graph neural networks.
  Our experiments show that the predictive performance of RuleGNNs is comparable to state-of-the-art graph classifiers using simple rules based on Weisfeiler-Leman labeling and pattern counting.
  Moreover, we introduce new synthetic benchmark graph datasets to show how to integrate expert knowledge into RuleGNNs making them more powerful than ordinary graph neural networks.
\end{abstract}

\section{Introduction}\label{sec:introduction}
Using expert knowledge to increase the efficiency, interpretability or predictive performance of a neural network is an evolving research direction in machine learning~\cite{DBLP:journals/ai/TowellS94,DBLP:journals/tkde/RudenMBGGHKPPRW23}.
Many ordinary neural network architectures are not capable of using external and structural information such as expert knowledge or meta-data, e.g., graph structures in a dynamic way.
We would like to motivate the importance of ``expert knowledge'' by considering the following example.
Maybe one of the best studied examples based on knowledge integration are convolutional neural networks~\cite{DBLP:conf/nips/CunBDHHHJ89}.
Convolutional neural networks for images use at least two extra pieces of ``expert knowledge'' that is:
\textit{neighbored pixels correlate}, and
\textit{the structure of images is homogeneous}.
The consequence of this \textit{knowledge} is the use of receptive fields and weight sharing.
It is a common fact that the usage of this information about images has highly improved the predictive performance over fully connected neural networks.
But what if expert knowledge suggests that rectangular convolutional kernels are not suitable to solve the task, or there exist two far away regions in the image that are important for the classification?
In this case the ordinary convolutional neural network architecture is too \textit{static} to adapt to the new information and thus dynamic neural networks are needed.
The above example illustrates the importance of the integration of expert knowledge into neural networks and should not be understood as the task to solve in this work.
In particular, dynamic neural networks are applicable to a wide range of domains including video~\cite{DBLP:conf/iccv/Wang0JSHH21}, text~\cite{DBLP:conf/iclr/JerniteGJM17} and graphs~\cite{DBLP:conf/cvpr/SimonovskyK17}.
The limitation of current approaches is that expert knowledge is somehow implicit and not directly encoded in the network structure, i.e., for each new information a new architecture has to be designed by hand.
Our goal is to extract the essence of dynamic neural networks by defining a new type of neural network layer that is on the one side able to use expert knowledge in a dynamic way and on the other side easily configurable using a simple general scheme.
Our solution to this problem are rule based layers that are able to encode expert knowledge directly in the network structure.
As far as we know, this is the first work that defines a dynamic neural network layer in this generality.

\paragraph{Main Idea} We simplify and unify the integration of expert knowledge and additional information into neural networks by proposing a two-step approach and show how to encode given extra information directly into the structure of a neural network in a dynamic way.
In the \textit{first step} the extra information or expert knowledge is formalized using appropriate rules (e.g., \textit{certain pixels or regions in images are important}, \textit{only nodes in a graph of type A and B interact}, \textit{some patterns, e.g., cycles or cliques, in a graph are important}, etc.).
In the \textit{second step} the rules are translated into rule functions that are
used to manipulate the structure of the neural network.
More precisely, each rule gives rise to a rule function that determines the positions of the weights in the weight matrix and the bias terms.
We note that the focus of this work is on the second step as we show how to use given rules to dynamically adapt the layers.
In fact, we do not provide a general instruction for deriving formal rules from given expert knowledge.
In contrast to ordinary network layers we consider a set $\weightset$ of learnable parameters instead of fixed weight matrices.
The weight matrices and bias terms are then constructed for each input sample independently using the learnable parameters from $\weightset$.
More precisely, each learnable parameter in $\weightset$ is associated with a specific relation between an input and output feature of a layer which increases the interpretability of our model.
\Cref{fig:weight-evaluation} shows the results of an application of our approach to learning on graphs.
In this case each input and output feature corresponds to a specific node in the graph.
The input samples are vectors (signals) corresponding to (a) molecule graphs respectively (b) snippets of social networks and the task is to predict the class of the given graph.
Each colored arrow in the figure corresponds to a learned parameter from $\weightset$, i.e., a specific relation between two atoms in the molecules or two nodes in the social network.
At the same time, it visualizes the message passing between the nodes in the graph.
Considering only the parameters with the $3$ largest positive and negative values, see the last column of (a) respectively (b), we see that our approach has learned to propagate information from outer atoms to the rings respectively from the nodes to the ``important'' nodes of the social network.
This example shows several advantages of our approach: (1) the architecture is very \textit{flexible} compared to classical architectures and allows to deal with \textit{different domains} and \textit{arbitrary input dimensions},
(2) messages can pass over \textit{arbitrary distances} in graphs in one layer, and
(3) the learned parameters and hence also the models are \textit{interpretable} and can be used to extract new knowledge from the data or to improve existing rules.

\begin{figure}[t]
    \centering
    \begin{subfigure}{0.48\textwidth}
    	\includegraphics[width=\textwidth]{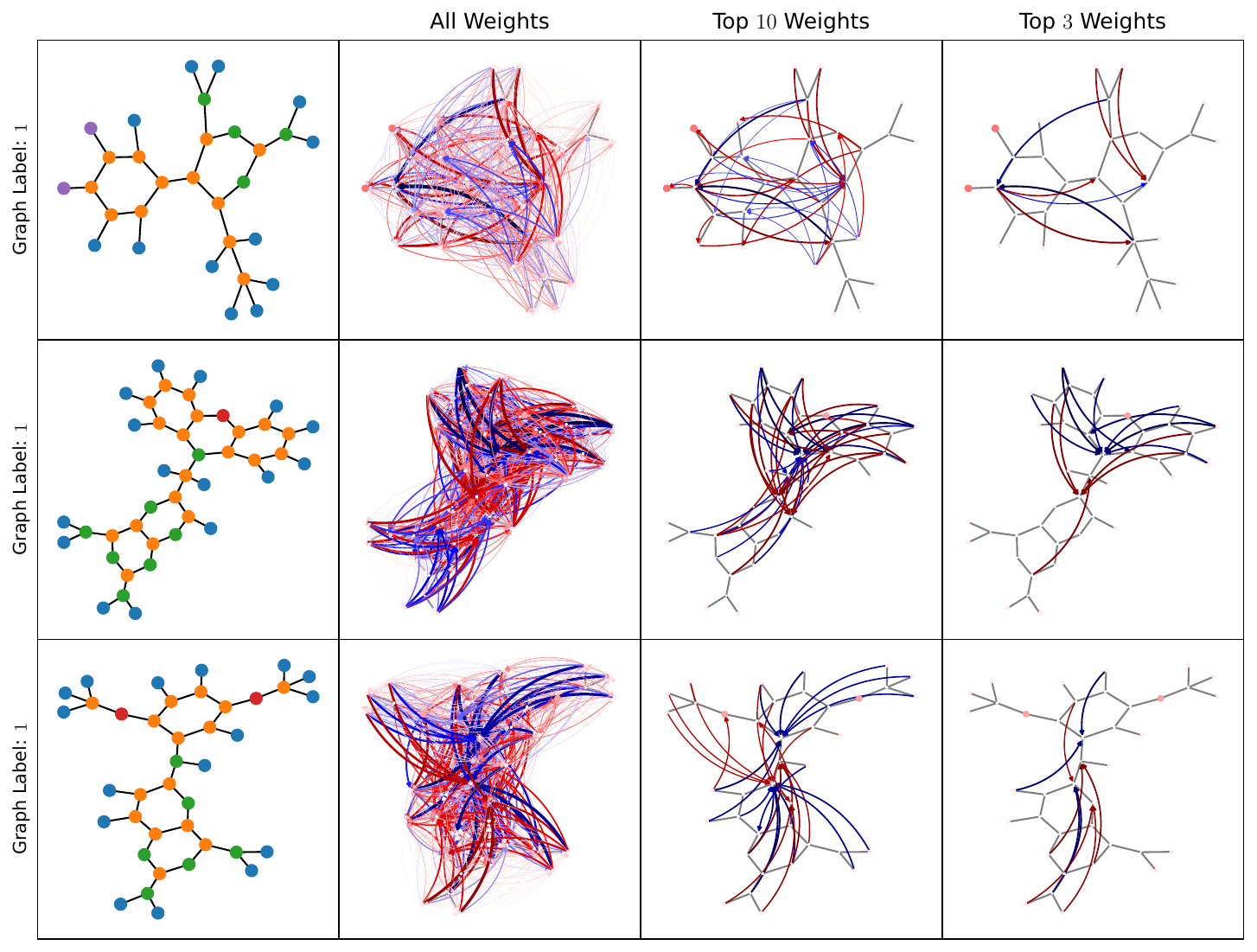}
        \caption{DHFR}
    \end{subfigure}\quad
    \begin{subfigure}{0.48\textwidth}
        \includegraphics[width=\textwidth]{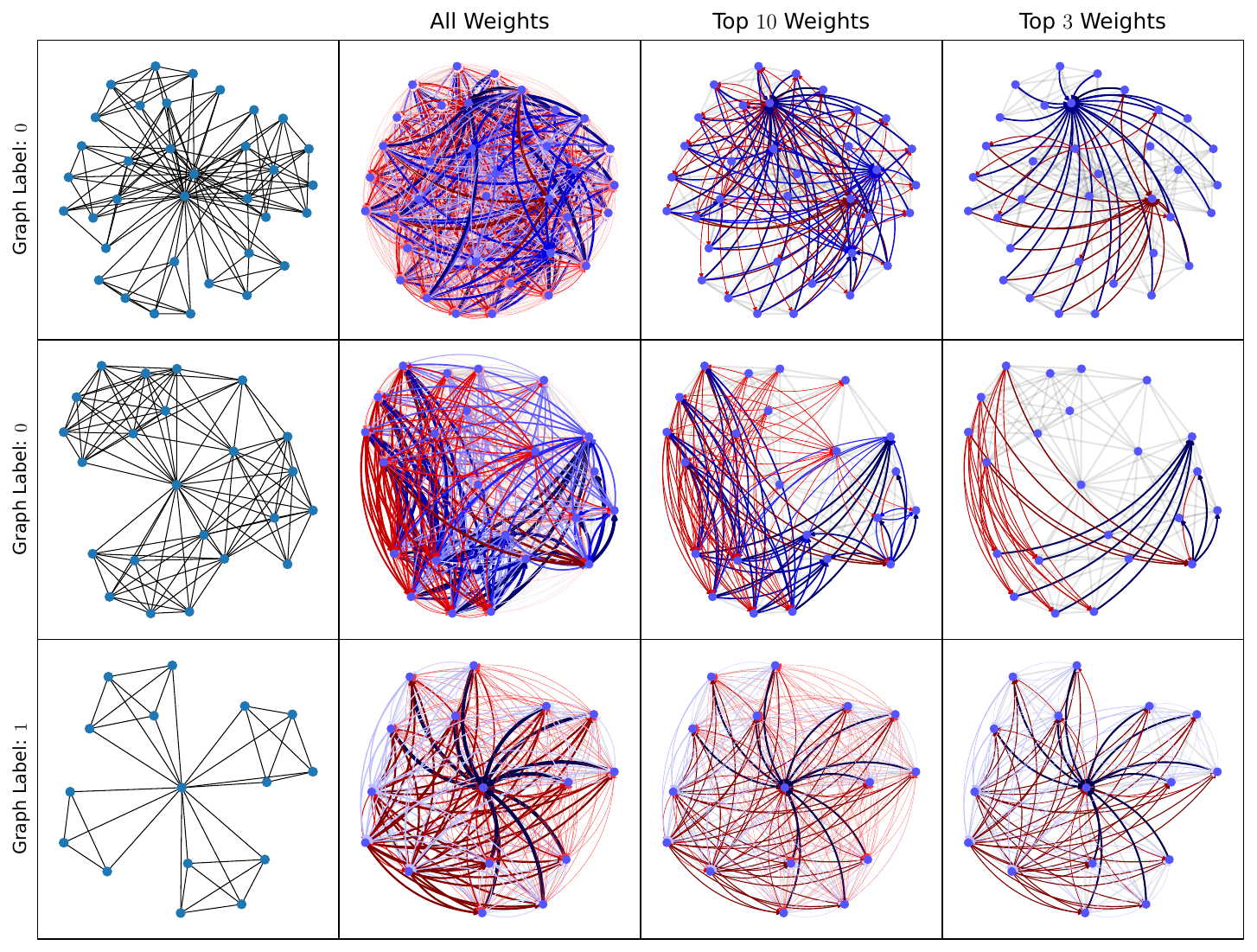}
        \caption{IMDB-BINARY}
    \end{subfigure}
    \caption{\label{fig:weight-evaluation} Visualization of the learned parameters of the best RuleGNN model on DHFR (a) and IMDB-BINARY (b) for three different random graphs from the test set.
    The label of the graph is given on the left side of the figure.
    Positive weights are denoted by red arrows and negative weights by blue arrows.
        The thickness and color corresponds to the absolute value of the weight.
        The size of the nodes corresponds to the bias values.
        The second to fourth columns of (a) resp. (b) show all, the $10$ and the $5$ largest positive and negative weights.}
\end{figure}

\paragraph{Main Contributions}
We define a new type of neural network layer called rule based layer.
This new layer can be integrated into arbitrary architectures making them dynamic, i.e., the structure of the network changes based on the input data and the predefined rules.
We prove that rule based layers generalize classical feed-forward layers such as fully connected and convolutional layers.
Additionally, we show that rule based layers can be applied to graph classification, by introducing RuleGNNs, a new type of graph neural networks.
In this way we are able to extend the concept of dynamic neural networks to graph neural networks together with all the advantages of dynamic neural networks.
RuleGNNs are by definition permutation equivariant, i.e., invariant to the order of the nodes in the graph
and able to handle graphs of arbitrary sizes.
Considering various real-world graph datasets, we demonstrate that RuleGNNs are competitive with state-of-the-art graph neural networks and other graph classification methods.
Using synthetic graph datasets we show that ``expert knowledge'' is easily integrable into our neural network architecture and also necessary for classification\footnote{See \href{https://github.com/fseiffarth/RuleGNNCode}{https://github.com/fseiffarth/RuleGNNCode}
and \href{https://github.com/fseiffarth/gnn-comparison}{https://github.com/fseiffarth/gnn-comparison}
for our code and results.}.

The rest of the paper is organized as follows: We introduce the concept of rule based layers in~\Cref{sec:RuleBasedLearning} and prove in Section~\ref{sec:TheoreticalAspects} that rule based layers generalize fully connected and convolutional layers.
In Section~\ref{sec:RuleBasedLearningOnGraphs} we present RuleGNNs and apply them in~\Cref{sec:Experiments} to different benchmark datasets and compare the results to state-of the art graph neural networks.
Finally, in~\Cref{sec:RelatedWork} we discuss related work, and in~\Cref{sec:Conclusion} some limitations and future work.

\section{Rule Based Learning}\label{sec:RuleBasedLearning}
Introducing the concept of rule based learning we first present some basic definitions followed by the formal definition of rule based layers.

\paragraph{Preliminaries}
For some $n\in\mathbb{N}$ we denote by $[n]$ the set $\{1,\ldots, n\}$.
A neural network is denoted by a function $\mathbf{f}(-, \weightset):\mathbb{R}^n\longrightarrow\mathbb{R}^m$ with the learnable parameters $\weightset$.
We extend this notation introducing an additional parameter $\ruleset$, that is the set of formal rules $\ruleset=\{\Rule^1,\ldots,\Rule^k\}$.
The exact definition of these rules is given in the next paragraph.
Informally, a rule $\Rule\in\ruleset$ is a function that determines the distribution of the weights in the weight matrix or the bias term of a layer.
A rule $\Rule$ is called \textit{dynamic} if it is a function in the input samples $x\in\mathbb{R}^n$ otherwise it is called \textit{static}.
An example of a static rule is the one used to define convolutional layers, see \Cref{prop:convolutional}.
An example of a dynamic rule can be found in Section~\ref{sec:RuleBasedLearningOnGraphs}.
A rule based neural network is a function $\mathbf{f}(-, \weightset, \ruleset):\mathbb{R}^*\longrightarrow\mathbb{R}^*$ that depends on a set of learnable parameters denoted by $\weightset$ and some rule set $\ruleset$ derived from expert knowledge or additional information.
The notation $*$ in the domain and co-domain of $\mathbf{f}$ indicates that the input and output can be of arbitrary or variable dimension.
As usual $\mathbf{f}$ is a concatenation of sub-functions $f^1, \ldots, f^l$ called the layers of the neural network.
More precisely, the $i$-th layer is a function $f^i(-,\weightset^i, \Rule^i):\mathbb{R}^{*}\longrightarrow\mathbb{R}^{*}$ where $\weightset^i$ is a subset of the learnable parameters $\weightset$ and $\Rule^i$ is an element of the ruleset $\ruleset$.
We call a layer $f^i$ \textit{static} if $\Rule^i$ is a static rule and \textit{dynamic} if $\Rule^i$ is a dynamic rule.
The input data is a triple $(\mathbf{D}, \mathbf{L}, \mathbf{I})$, where $\mathbf{D}=\{x_1\ldots, x_k\}$ with $x_i\in\mathbb{R}^*$ is the set of examples drawn from some unknown distribution.
The labels are denoted by $\mathbf{L}=(y_1\ldots, y_k)$ with $y_i\in\mathbb{R}^*$ and $\mathbf{I}$ is some additional information known about the input data $\mathbf{D}$, e.g., knowledge about the graph structure, node or edge labels or importance of certain regions in an image.
One main assumption of this paper is that $\mathbf{I}$ can be used to derive a set of static or dynamic rules $\ruleset$.
Again we would like to mention that we concentrate on the analysis of the effects of applying different rules $\Rule$ and not on the very interesting but also wide field of deriving the best rules $\ruleset$ from $\mathbf{I}$, see some discussion in Section~\ref{sec:Conclusion}.
Nonetheless, we always motivate the choice of the rules derived by $\mathbf{I}$.

\paragraph{Rule Based Layers}
In this section we will give a formal definition of a rule based layer.
Given some dataset $(\mathbf{D}, \mathbf{L}, \mathbf{I})$ defined as before and the rule set $\ruleset$ derived from $\mathbf{I}$, the task is to learn the weights $\weightset$ of the rule based neural network $\mathbf{f}$ to predict the labels of unseen examples drawn from an unknown distribution.
Our contribution concentrates on single layers and is fully compatible with other layers such as linear layers or convolutional layers
Hence, in the following we restrict to the $i$-th layer $f^i(-,\weightset^i, \Rule^i):\mathbb{R}^{*}\longrightarrow\mathbb{R}^{*}$ of a network $\mathbf{f}$.
For simplicity, we assume $i=1$ and omit the indices, i.e., we write $f\coloneqq f^i$, $\weightset\coloneqq\weightset^i$ and $\Rule\coloneqq\Rule^i$.
The forward propagation step of the rule based layer $f$ which will be a generalization of certain known layers as shown in Section~\ref{sec:TheoreticalAspects} is as follows.
Fix some input sample $x\in\mathbf{D}$ with $x\in \mathbb{R}^n$.
Then $f(-, \weightset, \Rule):\mathbb{R}^n\longrightarrow\mathbb{R}^m$ for $n, m\in\mathbb{N}$ is given by
\begin{align}\label{eq:RBL}f(x, \weightset, \Rule)=\sigma(W_{\Rule_W(x)}\cdot x+b_{\Rule_b(x)})\enspace .\end{align}
Here $\sigma$ denotes an arbitrary activation function and $W_{\Rule_W(x)}\in\mathbb{R}^{m\times n}$ rsp. $b_{\Rule_b(x)}\in\mathbb{R}^m$ is some weight matrix rsp. weight vector depending on the input vector $x$ and the rule $\Rule$.
The set $\weightset\coloneqq\{w_1,\ldots, w_N, b_1, \ldots, b_M\}$ consists of all possible learnable parameters of the layer.
The parameters $\{w_1,\ldots, w_N\}$ are possible entries of the weight matrix while $\{b_1,\ldots, b_M\}$ are possible entries of the bias vector.
The key point here is that the rule $\Rule$ determines the choices and the positions of the weights from $\weightset$ in the weight matrix $W_{\Rule_W(x)}$ and the bias vector $b_{\Rule_b(x)}$ depending on the input sample $x$.
In particular, \textit{not} all learnable parameters must be used in the weight matrix and the bias vector for some input sample $x$.
In contrast to ordinary neural network layers, the weight matrix and the bias vector are not fixed but \textit{functions} of the input sample $x$.
Moreover, for two samples $x, y\in\mathbf{D}$ of different dimensionality, e.g., $x\in\mathbb{R}^n$ and $y\in\mathbb{R}^k$ with $n\neq k$ the weight matrices $W_{\Rule_W(x)}$ and $W_{\Rule_W(y)}$ also have different dimensions and the learnable parameters can be in totally different positions in the weight matrix.

Given the set of learnable parameters $\weightset\coloneqq\{w_1,\ldots, w_N, b_1, \ldots, b_M\}$, for each input $x\in\mathbb{R}^n$ each rule $\Rule$ induces two rule functions

 \begin{align}
     \label{eq:RBF}\Rule_W(x): [m]\times[n]\longrightarrow \{0\} \cup [N] \enspace\enspace \text{ and } \enspace\enspace \Rule_b(x): [m]\longrightarrow \{0\} \cup [M]
 \end{align}

where $m\in\mathbb{N}$ is the output dimension of the layer that can also depend on $x$.
 In the following we abbreviate $\Rule_W(x)(i, j)$ by $\Rule_W(x, i, j)$ and $\Rule_b(x)(i)$ by $\Rule_b(x, i)$.
For simplicity, we assume that the matrix and vector indices start at $1$ and not at $0$.
 Using the associated rule functions~\eqref{eq:RBF} we can construct the weight matrix resp. bias vector by defining the entry $(i,j)\in\mathbb{R}^{m\times n}$ in the $i$-th row and the $j$-th column of the weight matrix $W_{\Rule(x)}\in\mathbb{R}^{m\times n}$ via
 \begin{align}\label{def:WeightMatrix}W_{\Rule_W(x)}(i, j):=
 \begin{cases}0& \text{if } \Rule_W(x, i, j) = 0\\
 w_{\Rule_W(x, i, j)}& \text{ o.w.}
 \end{cases}
\end{align}
and the entry at position $k$ in the bias vector $b_{\Rule_b(x)}\in\mathbb{R}^m$ by
 \begin{align}\label{def:BiasVector}b_{\Rule_b(x)}(k):=
	\begin{cases}0& \text{if } \Rule_b(x, k) = 0\\
		b_{\Rule_b(x, k)}& \text{ o.w.}
	\end{cases}.
\end{align}

Hence, an entry of the weight matrix or the bias vector of a rule based layer as defined in~\eqref{eq:RBL}
is zero if the value of the rule function is zero, otherwise the entry is the learnable parameter from the set $\weightset$ and the index is given by the rule function.
More precisely, the rule controls the connection between the $i$-th input and the $j$-th output feature in the weight matrix.
A rule $\Rule$ is called \textit{static} if the corresponding rule functions $\Rule_W$ and $\Rule_b$ are
independent of the input $x\in\mathbf{D}$, i.e., $\Rule_W(x)=\Rule_W(y)$ and $\Rule_b(x)=\Rule_b(y)$
for all inputs $x, y\in \mathbb{R}\in\mathbf{D}$ otherwise it is called \textit{dynamic}.
We call a rule based layer as defined in~\eqref{eq:RBL} \textit{static} if it is based on a static rule $\Rule$ and \textit{dynamic} otherwise.
We will show in \Cref{sec:TheoreticalAspects} that rule based layers generalize known concepts of neural network layers for specific rules $\Rule$.
In fact, we show that fully connected layers and convolution layers are static rule based layers.
Examples of dynamic rule based layers are given later on in Section~\ref{sec:RuleBasedLearningOnGraphs}.
The back-propagation of such a layer can be done as usual enrolling the computation graph of the forward step and applying iteratively the chain rule to all the computation steps.
We will not go into the details of this computation as it is similar to many other computations using backpropagation with shared weights.
For the experiments we use the automatic backpropagation tool of PyTorch~\cite{Paszke_PyTorch_An_Imperative_2019} which fully meets our requirements.

\paragraph{Assumptions and Examples}
Rule based learning relies on the following two main assumptions:
$A1)$ There is a connection between the additional information or expert knowledge $\mathbf{I}$ and the used rule $\Rule$ and
$A2)$ The distribution of weights given by the rule $\Rule$ in the weight matrix $W_{\Rule(x)}$ improves the predictive performance or increases the interpretability of the neural network.
As stated before we concentrate on the second assumption and consider different distribution of weights in the weight matrix given by different rules.
In fact, we assume without further consideration that it is possible to derive meaningful rules $\Rule$ from the additional information or expert knowledge $\mathbf{I}$.
For example if the dataset consists of images we can derive the ``informal'' rule that neighbored pixels are more important than pixels far away
and in case of chemical data there exists, e.g., the ortho-para rule for benzene rings that makes assumptions about the influence of atoms for specific positions regarding the ring.
This rule was already learned by a neural network in~\cite{Zhou2017GraphCA}.
It is another very interesting task which is beyond the scope of this work how to formalize these ``informal'' rules or to learn the ``best'' formal rules from the additional information $\mathbf{I}$.

In the following sections we focus on the concept of rule based layers and therefore for simplicity we only consider the rule function of weight matrices.
The rule function associated with the bias term can be constructed similarly.
For simplicity, we write $\Rule$ instead of $\Rule_W$.

\section{Theoretical Aspects of Rule Based Layers}\label{sec:TheoreticalAspects}
In this section we provide a theoretical analysis of rule based layers and show that they generalize fully connected and convolutional layers.
More precisely, we define two \textit{static} rules $\RuleFC$ and $\RuleCV$ and show that the rule based layer as defined in~\eqref{eq:RBL}
based on $\RuleFC$ is a fully connected layer and the rule based layer based on $\RuleCV$ is a convolutional layer.

\begin{prop}\label{prop:fullyconnected}
	Let $f(-, \weightset,\RuleFC):\mathbb{R}^{n}\longrightarrow\mathbb{R}^{m}$
with \[f(y, \weightset,\RuleFC)=\sigma(W_{\RuleFC(x)}\cdot y)\] be a rule based layer of a neural network as defined in~\eqref{eq:RBL} (without bias term) with learnable parameters $\weightset=\{w_1, \ldots, w_{n\cdot m}\}$ and $y=\mathbf{f}^{i}(x)$ is the result of the first $i-1$ layers.
Then for the rule function $\RuleFC(x):[m]\times[n]\rightarrow[m\cdot n]$ defined for all inputs $x$ as follows
	\[\RuleFC\coloneqq\RuleFC(x)(i, j)\coloneqq(i-1)\cdot n + j,\]
	the rule based layer $f$ is equivalent to a fully connected layer with activation function $\sigma$.
\end{prop}

\begin{proof}
To show the equivalence between the two layers it suffices to show that their weight matrices coincide.
	In case of fully connected layers we have to show that the weight matrix $W_{\RuleFC(x)}\in\mathbb{R}^{m\times n}$ is filled with $n\cdot m$ distinct weights.
	This can be easily checked by computing $W_{\RuleFC(x)}$ using the definition of the weight distribution based on the rule function in~\eqref{def:WeightMatrix}.
\end{proof}

\Cref{prop:fullyconnected} shows that rule based layers generalize fully connected layers of arbitrary size without bias vector and can be easily adapted to include the bias vector.
Hence, this shows that rule based layers generalize arbitrary fully connected layers.
Moreover, fully connected layers are static rule based layers as the rule $\RuleFC$ is static because it does not depend on the particular input $x$.

\begin{prop}\label{prop:convolutional}
	Let $f(-, \weightset,\RuleCV):\mathbb{R}^{n\cdot m}\longrightarrow\mathbb{R}^{(n-N + 1)\cdot(m-N + 1)}$
with \[f(y, \weightset,\RuleCV)=\sigma(W_{\RuleCV(x)}\cdot y)\] be a rule based layer of a neural network as defined in~\eqref{eq:RBL} (without bias term) with learnable parameters $\weightset=\{w_1, \ldots, w_{N^2}\}$ and $y=\mathbf{f}^{i}(x)$ is the result of the first $i-1$ layers.
Then for the rule function $\RuleCV:[(n-N+1)\cdot(m-N+1)]\times[n\cdot m]\rightarrow[N^2]$ defined by
	\[\RuleCV\coloneqq\RuleCV(x)(i, j)\coloneqq\begin{cases}
	\tau(i,j) & \mbox{if~~} 0 < \gamma(i, j) < N\cdot n \text{ and } \\
	& ~0 < j \Mod{n} - j + \gamma(i, j) < N\\
	0 & \text{o.w.}
	\end{cases}
	\]
\[\begin{array}{llll}
	  \renewcommand{\arraystretch}{1.5}
	  \text{with }&\tau(i,j)&=&\gamma(i, j)-((\gamma(i, j)-1)//n)\cdot(n-N)\\
	  \text{and }&\gamma(i, j)&=&j-((i-1)//(n-N + 1))\cdot n + (i-1)\Mod{(n-N + 1)}
\end{array}\]

	the rule based layer $f$ is equivalent to a convolution layer with quadratic kernel of size $N$ ($N\leq n$, $N\leq m$) and a stride of one over a two-dimensional image of size $n\times m$ (without padding and bias vector) with activation function $\sigma$.
The notation $a//b$ denotes the integer division.
\end{prop}
\begin{proof}
	Instead of the original two-dimensional image of size $n\times m$ we consider a reshaped vector $x\in\mathbb{R}^{n\cdot m}$ as our definition of rule based layers uses vector matrix multiplication.
	The output vector of dimension $(n-N + 1)\cdot(m-N + 1)$ can then again be reshaped into a two-dimensional image of size $(n-N + 1)\times(m-N + 1)$.
	Unfortunately, the reshaping makes the rule function complicated as the indices of the reshaped vector have to be mapped to the indices of the two-dimensional image.

	First note that convolution with a $N\times N$ kernel corresponds to matrix-vector multiplication of a doubly block circulant matrix that is a special case of a block Toeplitz matrix.
	Hence, what remains to show the equivalence between the layers is to compare the weight matrices and show that the entries in $W_{\RuleCV(x)}\in\mathbb{R}^{{(n-N + 1)\cdot(m-N + 1)}\times {n\cdot m}}$ exactly match the entries in the doubly block circulant matrix that corresponds to the convolution kernel.
	Indeed, using the definition of the doubly block circulant matrix that corresponds to the convolution kernel and compare it to the above given rule shows that the rule exactly returns the correct entries.
	Hence, the multiplication of $x$ with $W_{\RuleCV(x)}$ is equivalent to multiplication of $x$ with the doubly block circulant matrix that is equivalent to the convolution of $x$ with a kernel of size $N\times N$.
\end{proof}

	Proposition~\ref{prop:convolutional} shows that rule based layers generalize 2D-image convolution without padding and bias term.
	By adaption of the rule function it is possible to include the bias vector and padding.
	Moreover, the result can be generalized to higher dimensions kernels, non-quadratic kernels and arbitrary input and output channels.
	In fact, rule based layers can represent arbitrary shaped receptive fields changing the rule function accordingly.
	Hence, rule based layers generalize arbitrary convolutional layers.
	Convolutional layers are static rule based layers as the rule $\RuleCV$ is static because it is independent of the input.
	The following result is a direct implication from Propositions~\ref{prop:fullyconnected} and~\ref{prop:convolutional}.

\begin{thm}
	Rule based layers generalize fully connected and convolutional feed-forward layers.
	In particular, both layers are static rule based layers.
\end{thm}

We claim that also other types of feed-forward layers can be generalized by rule based layers using appropriate rule functions.

\section{Rule Based Learning on Graphs}\label{sec:RuleBasedLearningOnGraphs}
One of the main advantages of rule based layers as introduced in this work is that they give rise to a dynamic neural network architecture that is freely configurable using almost arbitrary rule functions.
In fact, the rule based neural networks can handle input samples independent of the dimension and structure.
Hence, a natural application of our approach is the task of graph classification.
We would like to emphasize that graph classification is only one of many possible applications of rule based layers.
Other possible applications are node classification, regression tasks, graph embeddings or applications on completely different domains like text or images.

\paragraph{Graph Preliminaries}
A graph is a pair $G=(V, E)$ with $V$ denoting the set of nodes of $G$ and $E\subseteq\{\{i,j\}\mid i,j\in V\}$ the set of edges.
All graphs are undirected and do not contain self-loops or parallel edges.
In case that it is clear from the context we omit $G$ and only use $V$ and $E$.
The distance between two nodes $i,j\in V$ in a graph, i.e., the length of the shortest path between $i$ and $j$, is denoted by $d(i,j)$.
A labeled graph is a graph $G=(V, E, l)$ equipped with a function $l:V\rightarrow \labelset$ that assigns to each node a label from the set $\labelset\subseteq\mathbb{N}$.
Two labeled graphs $G=(V, E, l)$ and $G'=(V', E', l')$ are isomorphic if there exists a bijection $\pi:V\rightarrow V'$ such that $\{i,j\}\in E$ if and only if $\{\pi(i), \pi(j)\}\in E'$ and $l(i)=l'(\pi(i))$ for all $i\in V$.
In this work the input samples corresponding to a graph $G=(V, E)$ are always vectors $x\in\mathbb{R}^{|V|}$.
In particular, the input vectors can be interpreted as signals over the graph and each dimension of the input vector corresponds to the one-dimensional input signal of a graph node

\subsection{Graph Rules}\label{subsec:graph-rules}
The example on molecule graphs in~\Cref{fig:RuleGNNExample} and~\Cref{subsec:example-molecule-graphs} motivates the intuition behind different graph specific rules that can be used to define a graph neural network based on rule layers.
Note that for $m=n=|V|$ the rule functions as defined in~\eqref{eq:RBF} can be interpreted as a mapping
from node pairs $(i,j)\in V\times V$  ($\Rule_W$) or nodes $i\in V$ ($\Rule_b$) to $0$ or the index of the learnable parameter in $\weightset$.
Two different node pairs $(i,j)$ and $(k,l)$ should map to the same integer if and only they ``behave similar'' in the graph.
Our starting point for a general scheme to define rule functions for graphs is a labeling function $\graphlabeling:V\rightarrow\labelset$.
In case of molecule graphs take for example the atom labels as node labels, see \Cref{subsec:example-molecule-graphs} as an example.
For unlabeled graphs it is possible to use the degree of a node.
Moreover, we use a property function $\propertyfunction:V\times V\rightarrow 0\cup\mathbb{N}$ that defines a relation between two nodes $i,j\in V$ in a graph.
Examples are the distance between two nodes, the type of edge connecting the nodes or the information that $i$ and $j$ are in one circle or not.
In this way we assign a triple $t(i,j)=(\graphlabeling(i), \graphlabeling(j), \propertyfunction(i,j))$ to each pair of nodes $(i,j)$.
Wit this preliminary work we can define $\Rule_W$ and $\Rule_b$ from~\eqref{eq:RBF} as follows.
We recall that the output of $\Rule_W$ maps each pair of nodes to some integer (index of the weight) or zero (no connection).
If a certain property is not fulfilled, e.g., if the distance between two nodes is too large or the type of the edge is invalid there should be no connection between the nodes.
Thus, for $\mathcal{D}\subseteq 0\cup\mathbb{N}$ being the set of all valid values for $\propertyfunction$
values we require $\Rule_W(i,j)=0$ if $p(i,j)\notin \mathcal{D}$.
Moreover, we require $\Rule_W(i,j)=\Rule_W(k,l)$ if and only if $t(i,j)=t(k,l)$ and $p(i,j)=p(k,l)\in \mathcal{D}$.
For the bias term we require that $\Rule_b$ maps each node to zero or some integer and two nodes $i$ and $j$ to the same integer if and only if $\graphlabeling(i)=\graphlabeling(j)$.
Besides these requirements the exact value of the integer is not important as it is only the index of the weight in the set of learnable parameters.
Of course in practice it is of advantage to use consecutive integers starting at $1$ for the indices.

\begin{prop}\label{prop:permutation-equivariance}
	Let $f(-, \weightset,\Rule)$ be a graph rule based layer and
	$x\in\mathbb{R}^{|V|}$ the input signal corresponding to a graph $G=(V,E)$
	Then for every permutation $\pi$ of the node order of $G$ (permutation of the entries of $x$) it holds $f(\pi(x), \weightset,\Rule)=\pi(f(x, \weightset,\Rule))$, i.e., $f$ is permutation equivariant.
\end{prop}
\begin{proof}
	Using the definition of rule based layers~\eqref{eq:RBL} we have that
	\[\pi(f(x, \weightset,\Rule))=\pi(\sigma(W_{\Rule_W(x)}\cdot x + b_{\Rule_b(x)}))\enspace.\]
	The entries of the weight matrix and the bias term are given by the rule functions $\Rule_W$ and $\Rule_b$, see~\eqref{eq:RBF}, which depend on the input signal $x$, i.e., the node order of the graph.
	Thus, by definition the permutation of the input signal $\pi(x)$ permutes the entries of the corresponding
	weight matrix and the bias term in the same way compared to the original input signal $x$.
	Therefore, the result of the multiplication of the permuted weight matrix with the permuted input signal is the same as the permutation of the result of the multiplication of the original weight matrix with the original input signal and it follows
	\[\pi(\sigma(W_{\Rule_W(x)}\cdot x + b_{\Rule_b(x)}))=\sigma(W_{\Rule_W(\pi(x))}\cdot \pi(x) + b_{\Rule_b(\pi(x))})=f(\pi(x), \weightset,\Rule)\]
	which completes the proof.
\end{proof}

In particular, \Cref{prop:permutation-equivariance} shows that each $l$ and $p$ as defined above gives rise to a permutation equivariant rule based layer.
Thus, finding a meaningful rule for graphs reduces to finding a meaningful labeling function $l$ and property function $p$.
In the following we focus on three different rule based layers that are based on well-known graph labeling functions.

\paragraph{Weisfeiler-Leman Layer}\label{par:wl-rule}
Recent research has shown that Weisfeiler-Leman labeling is a powerful tool for graph classification~\cite{DBLP:journals/jmlr/ShervashidzeSLMB11,DBLP:conf/aaai/0001RFHLRG19,DBLP:conf/icml/BodnarF0OMLB21,DBLP:conf/aaai/Truong024}.
Thus, we propose to use $1$-Weisfeiler-Leman labels of iteration $k$ as one option for $\graphlabeling$.
The $1$-Weisfeiler-Leman algorithm assigns in the $k$-th iteration to each node of a graph a label based on the structure of its local $k$-hop neighborhood, see~\cite{DBLP:journals/jmlr/ShervashidzeSLMB11} for the details\footnote{Usually, Weisfeiler-Leman labels are represented via strings.
For our purpose the strings are hashed to integers.
}.
For $p$ we use the distance between two nodes, i.e., $p\equiv d$ and $\mathcal{D}\subset 0\cup\mathbb{N}$ is the set of valid distances.
We denote the induced rule by $\Rule_{WL_{k, \mathcal{D}}}$.
For computational reasons in the experiments we restrict the maximum number of different Weisfeiler-Leman labels considered by some bound $L$.
We relabel the most frequent $L-1$ labels to $1,\cdots, L-1$ and set all other labels to $L$.
The corresponding layer is denoted by $f_{WL_{k,\mathcal{D},L}}$.

\paragraph{Pattern Counting Layer}
Beyond labeling nodes via the Weisfeiler-Leman algorithm, it is a common approach to use subgraph isomorphism counting to distinguish graphs~\cite{DBLP:journals/pami/BouritsasFZB23}.
This is in fact necessary as the 1-Weisfeiler-Leman algorithm is not able to distinguish some types of graphs, for example circular skip link graphs~\cite{DBLP:journals/combinatorica/CaiFI92} and strongly regular graphs~\cite{DBLP:conf/icml/BodnarF0OMLB21,DBLP:journals/pami/BouritsasFZB23}.
Thus, we propose a node labeling function $l$ based on pattern counting.
The function $p$ is the same as for the Weisfeiler-Leman layer.
In general, subgraph isomorphis counting is a hard problem~\cite{DBLP:conf/stoc/Cook71}, but for the real-world and synthetic benchmark graph datasets that are usually considered, subgraphs of size $k\in\{3,4,5,6\}$ can be enumerated in a preprocessing step in a reasonable time, see \Cref{tab:preprocessing-times}.
Given a set of patterns, say $\mathcal{P}$, we compute all possible embeddings of these patterns in the graph dataset in a preprocessing step.
Then for each pattern $P\in\mathcal{P}$ and each node $i\in V$ we count how often the node $i$ is part of an embedding of $P$.
Using those counts we define a labeling function $\graphlabeling:V\rightarrow\labelset\subseteq \mathbb{N}$ and
two nodes $i, j\in V$ are mapped to the same label if and only if their counts are equal for all patterns in $\mathcal{P}$.
Patterns that are often used in practice are small cycles, cliques, stars or paths.
We denote the corresponding rule by $\Rule_{\mathcal{P}_\mathcal{D}}$.
As for the Weisfeiler-Leman Rule we restrict the maximum number of different labels to some number $L$.
The corresponding layer is denoted by $f_{\mathcal{P}_{\mathcal{D},L}}$.

The total number of learnable parameters for a layer of type Weisfeiler-Leman or Pattern counting is bounded by $L\cdot L\cdot|\mathcal{D}|$ for the weight matrix and $|L|$ for the bias term.

\paragraph{Aggregation Layer}
In contrast to the above layers we assume that $m=M$ and $n=|V|$.
Let $\graphlabeling:V\rightarrow\labelset$ be an arbitrary labeling function, e.g., the atom labels in molecule graphs, the degree of the nodes or the Weisfeiler-Leman labels.
We require the rule function $\RuleFinal^{M}$ associated with the weight matrix to assign each pair $(n,i)$ with $i\in V$ and $n\in[M]$ an integer or zero based on $n$ and $\graphlabeling(i)$.
In fact, for each element of $\labelset$ the rule defines $M$ different learnable parameters.
The rule function $\RuleFinal^{M}$ associated with the bias is the identity, i.e., it represents an ordinary bias term with $M$ learnable parameters.
The corresponding layer is denoted by $f_{\RuleFinal^M}$.
We use this layer as output layer because its output is a fixed dimensional vector of size $M\in\mathbb{N}$ independent of the input size.

The total number of learnable parameters for the aggregation layer is bounded by $M\cdot|\labelset|$ for the weight matrix and $M$ for the bias term.

\begin{prop} The aggregation layer $f_{\RuleFinal^M}$ is permutation invariant, i.e., for any permutation $\pi$ of the nodes of $G=(V, E)$ with corresponding input signal $x$ it holds $f_{\RuleFinal^M}(\pi(x), \weightset,\Rule)=f_{\RuleFinal^M}(x, \weightset,\Rule)$.
\end{prop}
\begin{proof} Using the definitions of the aggregation rule,~\eqref{def:WeightMatrix} and~\eqref{def:BiasVector} it follows that node permutations permute the rows
    of the weight matrix and thus have no effect on the bias vector.
    In fact, permutations of the rows of the weight matrix do not change the result of the multiplication of the weight matrix with the input signal.
    Thus, the result of the aggregation layer is invariant under permutations of the nodes of the graph.
\end{proof}

\subsection{Rule Graph Neural Networks (RuleGNNs)}\label{subsec:rule-gnns}

The layers defined above are the building blocks of RuleGNNs.
Each RuleGNN is a concatenation of different rule based layers from type \textit{Weisfeiler-Leman} and \textit{Pattern Counting} with different parameters followed by an \textit{Aggregation Layer}.
The input of the network is a signal $x\in\mathbb{R}^{|V|}$ corresponding to a graph $G=(V, E)$.
We note that for simplicity we focus on one-dimensional signals but our approach also allows multi-dimensional signals, i.e., $x\in\mathbb{R}^{|V|\times d}$.
The output of the network is a vector of fixed size $M\in\mathbb{N}$ determined by the aggregation rule where $M$ is usually the number of classes of the graph classification task.
The output can be also used as an intermediate vectorial representation of the graph or for regression tasks.
Note that RuleGNNs can be also used for node classification tasks by setting $M=|V|$ or by omitting the aggregation layer.

\begin{thm}[Expressive Power of RuleGNNs]\label{thm:expressivity-rulegnn}
    For each pair of non-isomorphic graphs $G$ and $G'$ there exists a RuleGNN $f(-,\weightset,\ruleset)$
	that can distinguish $G$ and $G'$.
\end{thm}
\begin{proof}
    The expressive power of the RuleGNNs is based on the expressive power of the underlying labeling function $\graphlabeling$.
	Indeed, we will show that a RuleGNN is at least as powerful as the labeling function $\graphlabeling$.
    Let $l$ be a labeling function that can distinguish $G$ and $G'$ by counting the occurrences of the labels, e.g., the $(k+1)$-WL labels where $k$ is the maximum of the treewidths of $G$ and $G'$~\cite{DBLP:journals/jgt/Dvorak10}.
    Now, consider the RuleGNN that consists only of the aggregation layer $f_{\RuleFinal^M}$ with $M=1$ based on the labeling function $l$.
    Without loss of generality we assume that each entry of the input signals $x$ resp. $y$ corresponding to $G$ resp. $G'$ is equal to $1$.
    Then $f_{\RuleFinal^M}(x)$ resp. $f_{\RuleFinal^M}(y)$ is equal to the sum of the learnable parameters corresponding to the labels of the nodes of $G$ resp. $G'$.
    By assumption $l$ can distinguish $G$ and $G'$ by counting the occurrences of the labels and hence also the above defined RuleGNN can distinguish $G$ and $G'$.
\end{proof}

\subsection{Example: RuleGNNs for Molecule Graphs}\label{subsec:example-molecule-graphs}
	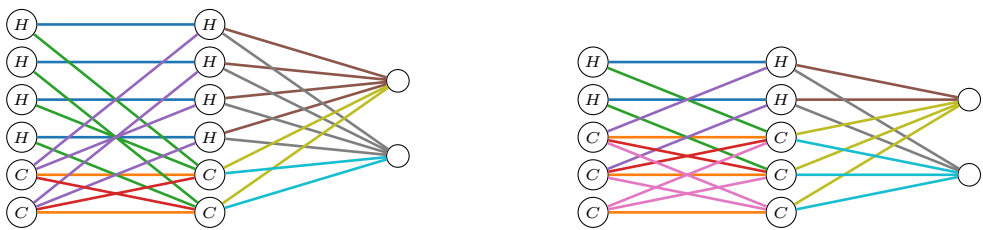
\begin{figure}[t]
		\centering

				\scalebox{1}[1]{
		\begin{subfigure}{0.45\textwidth}
	\begin{tikzpicture}[scale = 0.5]\small
	\node[circle,inner sep=0pt, minimum size=0.4cm, draw=black] (a) at (0,0) {\tiny $H$};
	\node[circle,inner sep=0pt, minimum size=0.4cm, draw=black] (b) at (0,-1) {\tiny $H$};
	\node[circle,inner sep=0pt, minimum size=0.4cm, draw=black] (c) at (0,-2) {\tiny $H$};
	\node[circle,inner sep=0pt, minimum size=0.4cm, draw=black] (d) at (0,-3) {\tiny $H$};
	\node[circle,inner sep=0pt, minimum size=0.4cm, draw=black] (e) at (0,-4) {\tiny $C$};
	\node[circle,inner sep=0pt, minimum size=0.4cm, draw=black] (f) at (0,-5) {\tiny $C$};

	\node[circle,inner sep=0pt, minimum size=0.4cm, draw=black] (g) at (5,0) {\tiny $H$};
	\node[circle,inner sep=0pt, minimum size=0.4cm, draw=black] (h) at (5,-1) {\tiny $H$};
	\node[circle,inner sep=0pt, minimum size=0.4cm, draw=black] (i) at (5,-2) {\tiny $H$};
	\node[circle,inner sep=0pt, minimum size=0.4cm, draw=black] (j) at (5,-3) {\tiny $H$};
	\node[circle,inner sep=0pt, minimum size=0.4cm, draw=black] (k) at (5,-4) {\tiny $C$};
	\node[circle,inner sep=0pt, minimum size=0.4cm, draw=black] (l) at (5,-5) {\tiny $C$};

	\draw[tab_blue, line width = 1pt] (a)--(g) (b)--(h) (c)--(i)  (d)--(j) ;
	\draw[tab_orange, line width = 1pt] (e)--(k) (f)--(l);
	\draw[tab_green, line width = 1pt] (a)--(k) (b)--(l) (c)--(k) (d)--(l);
	\draw[tab_purple, line width = 1pt]  (e)--(i) (e)--(g) (f)--(j) (f)--(h);
	\draw[tab_red, line width = 1pt]    (e)--(l) (f)--(k) ;

	\node[circle,inner sep=0pt, minimum size=0.3cm, draw=black] (y1) at (10,-1.5) {};
	\node[circle,inner sep=0pt, minimum size=0.3cm, draw=black] (y2) at (10,-3.5) {};

	\draw[tab_brown, line width = 1pt]  (y1)--(g) (y1)--(h) (y1)--(i) (y1)--(j);
	\draw[tab_gray, line width = 1pt]  (y2)--(g) (y2)--(h) (y2)--(i) (y2)--(j);
	\draw[tab_olive, line width = 1pt]    (y1)--(l) (y1)--(k);
	\draw[tab_cyan, line width = 1pt]    (y2)--(l) (y2)--(k);

	\end{tikzpicture}
	\end{subfigure}}\hfill
						\scalebox{1}[1]{
	\begin{subfigure}{0.45\textwidth}
	\begin{tikzpicture}[scale = 0.5]\small
	\node[circle,inner sep=0pt, minimum size=0.4cm, draw=black] (a) at (0,0) {\tiny $H$};
	\node[circle,inner sep=0pt, minimum size=0.4cm, draw=black] (b) at (0,-1) {\tiny $H$};
	\node[circle,inner sep=0pt, minimum size=0.4cm, draw=black] (c) at (0,-2) {\tiny $C$};
	\node[circle,inner sep=0pt, minimum size=0.4cm, draw=black] (d) at (0,-3) {\tiny $C$};
	\node[circle,inner sep=0pt, minimum size=0.4cm, draw=black] (e) at (0,-4) {\tiny $C$};

	\node[circle,inner sep=0pt, minimum size=0.4cm, draw=black] (g) at (5,0) {\tiny $H$};
	\node[circle,inner sep=0pt, minimum size=0.4cm, draw=black] (h) at (5,-1) {\tiny $H$};
	\node[circle,inner sep=0pt, minimum size=0.4cm, draw=black] (i) at (5,-2) {\tiny $C$};
	\node[circle,inner sep=0pt, minimum size=0.4cm, draw=black] (j) at (5,-3) {\tiny $C$};
	\node[circle,inner sep=0pt, minimum size=0.4cm, draw=black] (k) at (5,-4) {\tiny $C$};

	\draw[tab_blue,solid, line width = 1pt] (a)--(g);
	\draw[tab_blue,solid, line width = 1pt] (b)--(h);
	\draw[tab_orange, line width = 1pt] (c)--(i);
	\draw[tab_orange, line width = 1pt] (d)--(j);
	\draw[tab_orange, line width = 1pt] (e)--(k);
	\draw[tab_green, line width = 1pt] (a)--(i);
	\draw[tab_green, line width = 1pt] (b)--(j);
	\draw[tab_purple, line width = 1pt]  (c)--(g);
	\draw[tab_purple, line width = 1pt] (d)--(h);
	\draw[tab_red, line width = 1pt] (c)--(j);
	\draw[tab_red, line width = 1pt] (d)--(i);

	\draw[tab_pink, line width = 1pt] (c)--(k);
	\draw[tab_pink, line width = 1pt] (d)--(k);
	\draw[tab_pink, line width = 1pt] (e)--(i);
	\draw[tab_pink, line width = 1pt] (e)--(j);

	\node[circle,inner sep=0pt, minimum size=0.3cm, draw=black] (y1) at (10,-1) {};
	\node[circle,inner sep=0pt, minimum size=0.3cm, draw=black] (y2) at (10,-3) {};

	\draw[tab_brown, line width = 1pt]  (y1)--(g) (y1)--(h);
	\draw[tab_gray, line width = 1pt]  (y2)--(g) (y2)--(h);
	\draw[tab_olive, line width = 1pt]   (y1)--(k) (y1)--(i) (y1)--(j);
	\draw[tab_cyan, line width = 1pt]   (y2)--(k) (y2)--(i) (y2)--(j);
	\end{tikzpicture}
\end{subfigure}}
\caption{\label{fig:RuleGNNExample}
	Information propagation in a simple two layer RuleGNN based on the molecule graphs of ethylene (left) and cyclopropenylidene (right) and the rules $\RuleMol$ \eqref{eq:RuleMol} and $\RuleFinal^k$ \eqref{eq:RuleMolFinal}.
	The input signal is propagated from left to right. The graph nodes represent the neurons of the neural network.
	Edges of the same color denote shared weights in a layer.
}
\end{figure}
Assume the task is to learn a property of a molecule based on its graph structure.
In this example we present a RuleGNN that is a concatenation of two very simple rule based layers.
The advantage of rule based layers and hence also RuleGNNs is that they encode the graph structure (in this example the structure of two molecules) directly into the neural network.
Moreover, the input data can be arbitrary molecule graphs and the output is a vector of fixed size $k=2$ that encodes the property of the molecule or some intermediate vectorial representation.
In this example we consider the molecule graphs of ethylene and cyclopropenylidene given in Figure~\ref{fig:MoleculeGraph} together with their corresponding input signals $x\in\mathbb{R}^6$ and $y\in\mathbb{R}^5$.
The atoms of the molecules (hydrogen $H$ and carbon $C$) correspond to the nodes of a graph and the bonds to the edges.
The atom labels and the bond types  (\textit{single} and \textit{double}) can be seen as additional information $\mathbf{I}$ that is known about the input samples.
The graph nodes are indexed via integers in some arbitrary but fixed order and the atoms corresponding to the graph nodes are given by the labeling function $l:V\rightarrow\{H, C\}$.

\begin{figure}[t]
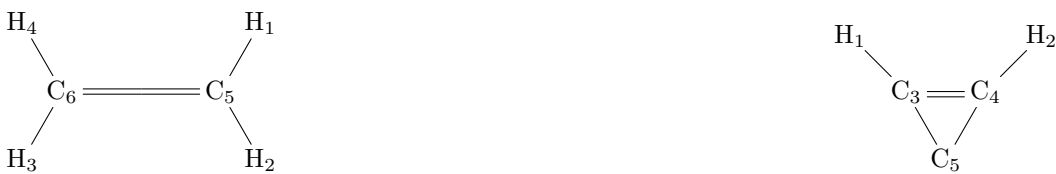
\centering
	\chemfig{{C_6}(==[:0]{C_5}(-[:60]{H_1})(-[:-60]{H_2}))(-[:-120]{H_3})(-[:120]{H_4})}
	\hfill
	\chemfig{[:-30]{C_3}(-[:135]{H_1})*3(-{C_5}-{C_4}(-[:45]{H_2})=)}
	\caption{\label{fig:MoleculeGraph}Molecule graphs of ethylene (left) and cyclopropenylidene (right). The indices denote the order of the nodes.}
\end{figure}

The RuleGNN consists of two rule based layers $f_1(-,\weightset_1,\RuleMol)$ and $f_2(-,\weightset_2,\RuleFinal^2)$ with learnable parameters $\weightset_1=\{w_1, \ldots, w_6\}$ and $\weightset_2=\{w'_1, \ldots, w'_4\}$
and the following rule functions $\RuleMol$ and $\RuleFinal^2$.
For some graph $G=(V, E)$ and its corresponding input signal $z$ we define $\RuleMol$ as follows:
\begin{align}\label{eq:RuleMol}
	\begin{array}{cccl}
\RuleMol(z):&[|V|]\times [|V|]&\longrightarrow& \{0\}\cup[6]\\\\
&(i, j)&\mapsto&\begin{cases}
1& \text{if } i=j \text{ and } l(i)=H\\
2&\text{if } i=j \text{ and }  l(i)=C \\
3&\text{if } (i,j) \text{ is an edge (-), } l(i)=H, l(j)=C\\
4&\text{if } (i,j) \text{ is an edge (-),} l(i)=C, l(j)=H\\
5&\text{if } (i,j) \text{ is an edge (-),} l(i)=l(j)=C\\
6&\text{if } (i,j) \text{ is an edge (=),} l(i)=l(j)=C\\
0& \text{o.w.}
\end{cases}
\end{array}
\end{align}

For some graph $G=(V, E)$ and its corresponding input signal $z$ we define $\RuleFinal$ as follows:
\begin{align}\label{eq:RuleMolFinal}
	\begin{array}{cccl}
\RuleFinal^2(z):&[2]\times [|V|]&\longrightarrow& \{0\}\cup[4]\\\\
&(i, j)&\mapsto&\begin{cases}
 i&  l(j)=H\\
 i+2&  l(j)=C \\
0& \text{o.w.}
\end{cases}
\end{array}
\end{align}

Note that $\RuleMol$ and $\RuleFinal^2$ are not restricted to the two molecules from above but can be applied to arbitrary molecule graphs.
Indeed, applying it to molecules with atom labels different from $H$ or $C$ makes the rules less powerful, i.e., it should be adapted to the type of molecules.
Using the definition~\eqref{def:WeightMatrix} of weight distribution defined by the rule function we can construct the weight matrices $W_{\RuleMol(x)}, W_{\RuleFinal^2(x)}$ for the ethylene graph and $W_{\RuleMol(y)}, W_{\RuleFinal(y)}$ for the cyclopropenylidene graph as follows:

\begin{table}[h]
\begin{tabular}{cc}
	$W_{\RuleMol(x)}= \begin{psmallmatrix}
	w_1&0&0&0&w_3&0\\
	0&w_1&0&0&w_3&0\\
	0&0&w_1&0&0&w_3\\
	0&0&0&w_1&0&w_3\\
	w_4&w_4&0&0&w_2&w_5\\
	0&0&w_4&w_4&w_5&w_2\\
	\end{psmallmatrix}$
	&

		$W_{\RuleFinal^2(x)}= \begin{psmallmatrix}
	w'_1&w'_1&w'_1&w'_1&w'_3&w'_3\\
	w'_{2}&w'_{2}&w'_{2}&w'_{2}&w'_{4}&w'_{4}\\
	\end{psmallmatrix}$
		\\&\\
	$W_{\RuleMol(y)}= \begin{psmallmatrix}
	w_1&0&w_3&0&0\\
	0&w_1&0&w_3&0\\
	w_4&0&w_2&w_6&w_5\\
	0&w_3&w_6&w_2&w_5\\
	0&0&w_5&w_5&w_2\\
	\end{psmallmatrix}
	$&
	$W_{\RuleFinal^2(y)}= \begin{psmallmatrix}
	w'_1&w'_1&w'_3&w'_3&w'_3\\
	w'_{2}&w'_{2}&w'_{4}&w'_{4}&w'_{4}\\
	\end{psmallmatrix}$

\end{tabular}
\end{table}

Combining the two rule based layers we obtain the RuleGNN and the forward propagation is given by $\sigma(W_{\RuleFinal(x)}\cdot \sigma(W_{\RuleMol(x)}\cdot x))$ for the ethylene graph and $\sigma(W_{\RuleFinal(y)}\cdot \sigma(W_{\RuleMol(y)}\cdot y))$ for the cyclopropenylidene graph.

Note that the forward propagation of the layer corresponding to the rule $\RuleMol$ is kind of a multiplication with a weighted adjacency matrix of the graph where the weights of the adjacency matrix are given by the learnable parameters, see also \Cref{fig:RuleGNNExample}.
In contrast to adjacency matrices the weight matrix is not necessary symmetric.
The computation graph induced by the weight matrices exactly represent the graph structure while the edge weights are shared across the network using the rule, see \Cref{fig:RuleGNNExample}.
Note that also edge labels (e.g., atomic bonds) can be taken into account by increasing the size of the weight set.
Moreover, it is possible to include bigger neighborhoods, i.e., all nodes reachable by $k$-hops.
Of course using other information of the graph (e.g., substructures (such as circles or cliques), node degrees, connections not depicted by edges) more complicated rules such as the Weisfeiler-Leman rule and Pattern Counting rules can be used.

\section{Experiments}\label{sec:Experiments}
We evaluate the performance of RuleGNNs on different real-world and synthetic benchmark graph dataset and compare the results to state-of-the-art algorithms .
For comparability and reproducibility of the results, we make use of the experimental setup from~\cite{Errica2019AFC}\footnote{See \href{https://github.com/fseiffarth/gnn-comparison}{https://github.com/fseiffarth/gnn-comparison} for the results of the state-of-the-art algorithms.}.
For each graph dataset we perform a $10$-fold cross validation, i.e., we use fixed splits\footnote{See \href{https://github.com/fseiffarth/RuleGNNCode}{https://github.com/fseiffarth/RuleGNNCode} for the data splits.} of the dataset into $10$ equally sized parts, and use $9$ of them for training, parameter tuning and validation.
We then use the model that performs best on the validation set and report the performance on the previously unseen test set.
We average three runs of the best model to decrease random effects.
The standard deviation reported in the tables is computed over the results on the $10$ folds.

\paragraph{Data and Competitors Selection}
A problem of several heavily used graph benchmark datasets like MUTAG or PTC~\cite{TUDortmund} is that node and edge labels seems to be more important than the graph structure itself, i.e., there is no significant improvement over simple baselines~\cite{Schulz2019OnTN}.
Moreover, in case of MUTAG the performance of the model is highly dependent on the data split because of the small number of samples.
Thus, in this work for benchmarking we choose DHFR, Mutagenicity, NCI1, NCI109, IMDB-BINARY and IMDB-MULTI from~\cite{TUDortmund} because the structure of the graphs seems to play an important role, i.e., simple baselines~\cite{Errica2019AFC,Schulz2019OnTN} are significantly worse than more evolved algorithms.
Additionally, we consider circular skip link graphs CSL~\cite{DBLP:journals/combinatorica/CaiFI92} and some new synthetic benchmark graph datasets called LongRings, EvenOddRings and Snowflakes \cite{naik2024iterative} to show that RuleGNNs can overcome limitations of ordinary graph neural networks.
For more details on the datasets see \Cref{subsec:details-on-the-datasets}.
For NCI1, IMDB-BINARY and IMDB-MULTI we use the same splits as in~\cite{Errica2019AFC} and for CSL we use the splits as in~\cite{DBLP:journals/jmlr/DwivediJL0BB23} and a $5$-fold cross validation.
We evaluate the performance of the RuleGNNs on these datasets and compare the results to the baselines from~\cite{Errica2019AFC} and~\cite{Schulz2019OnTN} and the Weisfeiler-Leman subtree kernel (WL-Kernel)~\cite{DBLP:journals/jmlr/ShervashidzeSLMB11} which is one of the best performing graph classification algorithm besides graph neural networks.
For comparison with state-of-the-art graph classification algorithms we follow~\cite{Errica2019AFC} and compare to DGCNN~\cite{DBLP:conf/aaai/ZhangCNC18}, GIN \cite{DBLP:conf/iclr/XuHLJ19} and GraphSAGE \cite{Hamilton2017InductiveRL}.
Additionally, we compare to the results of some recent state-of-the-art graph classification algorithms~\cite{DBLP:conf/nips/BodnarFOWLMB21,DBLP:conf/icml/BodnarF0OMLB21,DBLP:journals/pami/BouritsasFZB23,DBLP:conf/aaai/Truong024}.
For the latter we use the results from the respective papers that might be obtained with different splits of the datasets and another evaluation setup.

\paragraph{Experimental Settings and Resources}
All experiments were conducted on a AMD Ryzen 9 7950X 16-Core Processor with $128$ GB of RAM.
For the competitors we use the implementations from~\cite{Errica2019AFC}\footnote{See \href{https://github.com/fseiffarth/gnn-comparison}{https://github.com/fseiffarth/gnn-comparison} for the code.}.
For the real-world datasets we tested different rules and combinations of the layers defined in~\Cref{subsec:graph-rules}.
More details on the tested hyperparameters can be found in~\Cref{tab:hyperparameters}.
We always use tanh for activation and the Adam optimizer~\cite{DBLP:journals/corr/KingmaB14} with a learning rate of $0.05$ (real-world datasets) resp. $0.1$ (synthetic datasets).
For the real-world datasets the learning rate was decreased by a factor of $0.5$ after each $10$ epochs.
For the loss function we use the cross entropy loss.
All models are trained for $50$ (real-world) resp. $200$ (synthetic) epochs and the batch size was set to $128$.
We stopped the training if the validation accuracy did not improve for $25$ epochs.

\subsection{Dataset Details}\label{subsec:details-on-the-datasets}
In this section we provide additional details on the datasets used in the experiments.
Tables~\ref{tab:real-world-datasets} and~\ref{tab:synthetic-datasets} provide an overview of the real-world and synthetic datasets.
We consider four synthetic datasets.
The CSL dataset is from~\cite{DBLP:journals/combinatorica/CaiFI92}.
We constructed the others to test the ability for detecting long range dependencies and the expressive power beyond the 1-WL test.

\begin{table}[t]
	\centering
\tiny
    \begin{tabular}{lr|rrr|rrr|rrr|rr}
        \toprule
        Dataset & \#Graphs &  \multicolumn{3}{c}{\#Nodes} & \multicolumn{3}{c}{\#Edges} & \multicolumn{3}{c}{Diameter} & \#Node Labels & \#Classes \\
         &  & max & avg & min & max & avg & min & max & avg & min & & \\
        \midrule

        NCI1 & 4\,110 & 111 & 29.9 & 3 & 119 & 32.3 & 2 & 45 & 11.5 & 0 & 37 & 2 \\
        NCI109 & 4\,127 & 111 & 29.7 & 4 & 119 & 32.1 & 3 & 61 & 11.3 & 0 & 38 & 2 \\
        Mutagenicity & 4\,337 & 417 & 30.3 & 4 & 112 & 30.8 & 3 & 41 & 6.3 & 0 & 14 & 2 \\
        DHFR & 756 & 71 & 42.4 & 20 & 73 & 44.5 & 21 & 22 & 14.6 & 8 & 9 & 2 \\
        IMDB-BINARY & 1\,000 & 136 & 19.8 & 12 & 1249 & 96.5 & 26 & 2 & 1.9 & 1 & 1 & 2 \\
        IMDB-MULTI & 1\,500 & 89 & 13.0 & 7 & 1467 & 65.9 & 12 & 2 & 1.5 & 1 & 1 & 3 \\
        \bottomrule
    \end{tabular}
    \caption{\label{tab:real-world-datasets}
    Details of the real-world datasets \cite{TUDortmund} used in the experiments.}
\end{table}

\paragraph{LongRings} The dataset consists of $1200$ cycles of $100$ nodes each and is designed to test the ability to detect long range dependencies.
Four of the cycle nodes are labeled by $1, 2, 3, 4$ and all others by $0$.
The distance between each pair of the four nodes is exactly $25$ or $50$.
The label of the graph is $0$ if $1$ and $2$ have distance $50$, $1$ if $1$ and $3$ have distance $50$ and $2$ if $1$ and $4$ have distance $50$.
There are $400$ graphs per class.
The difficulty of the classification task is that information has to be propagated over a long distance.
Regarding RuleGNNs this is very easy as we can define an appropriate rule.

\paragraph{EvenOddRings} The dataset consists of $1200$ cycles of $16$ nodes each and is designed to test the ability to encode expert knowledge in the neural network architecture.
The nodes in each graph are labeled from $0$ to $15$.
The graph label is determined by labels of the nodes that have distance $8$ respectively $4$ to the node with label $0$.
We denote them by $x$ resp. $y, z$.
We have four cases: $x$ is even and $y+z$ is even, $x$ is even and $y+z$ is odd, $x$ is odd and $y+z$ is even, $x$ is odd and $y+z$ is odd.
There are $300$ graphs per class, i.e., each of the four cases.
The expert knowledge we use is that the information has to be collected from nodes of distance $8$ and $4$ only.

\paragraph{EvenOddRingsCount} The dataset consists of the same graphs as EvenOddRings but the graph labels are different.
For all nodes and their opposite node (distance 8) in the circle the sum of the labels is computed.
If there are more even sums than odd sums the graph is labeled by $0$ and by $1$ otherwise.
There are $600$ graphs per class.
The expert knowledge we use is the information that only distance $8$ is relevant.

\paragraph{Snowflakes} The dataset consists of graphs proposed by~\cite{naik2024iterative} that are not distinguishable by the 1-WL test, see~\Cref{fig:snowflakes-plot} for an example.
The dataset consists of circles of length $3$ to $12$ and at each circle node a graph from $M_0, M_1, M_2$ or $M_3$ is attached, see~\Cref{fig:m-plot} and~\cite{naik2024iterative} for the details.
$M_0, M_1, M_2$ and $M_3$ are non-isomorphic graphs that are not distinguishable by the 1-WL test.
One node in the circle is labeled by $1$ and all other nodes are labeled by $0$.
The label of the graph is determined by the graph $M_0, M_1, M_2$ or $M_3$ that is attached to the circle node with label $1$.

\begin{figure}[t]
	\centering
    {\includegraphics[width=\textwidth]{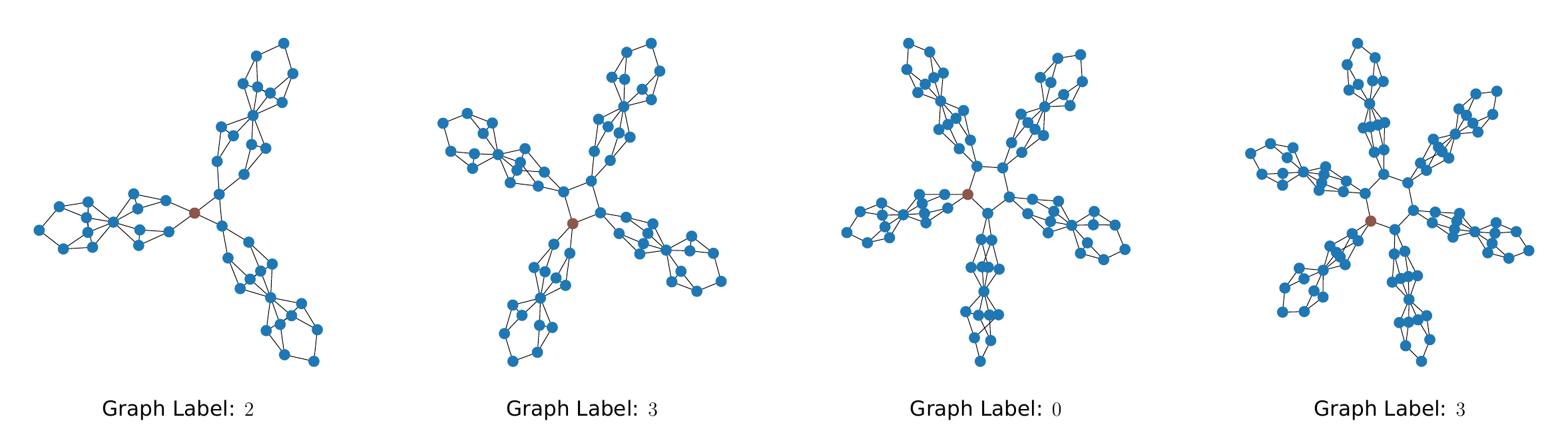}}
    \caption{\label{fig:snowflakes-plot}
    Example graphs from the \textit{Snowflakes} dataset. The brown node in the circle is labeled by $1$ and the other nodes by $0$.
    The label of the graph is determined by the subgraph attached to the brown node.
    }
\end{figure}

\begin{figure}[t]\centering
    {\includegraphics[width=0.5\textwidth]{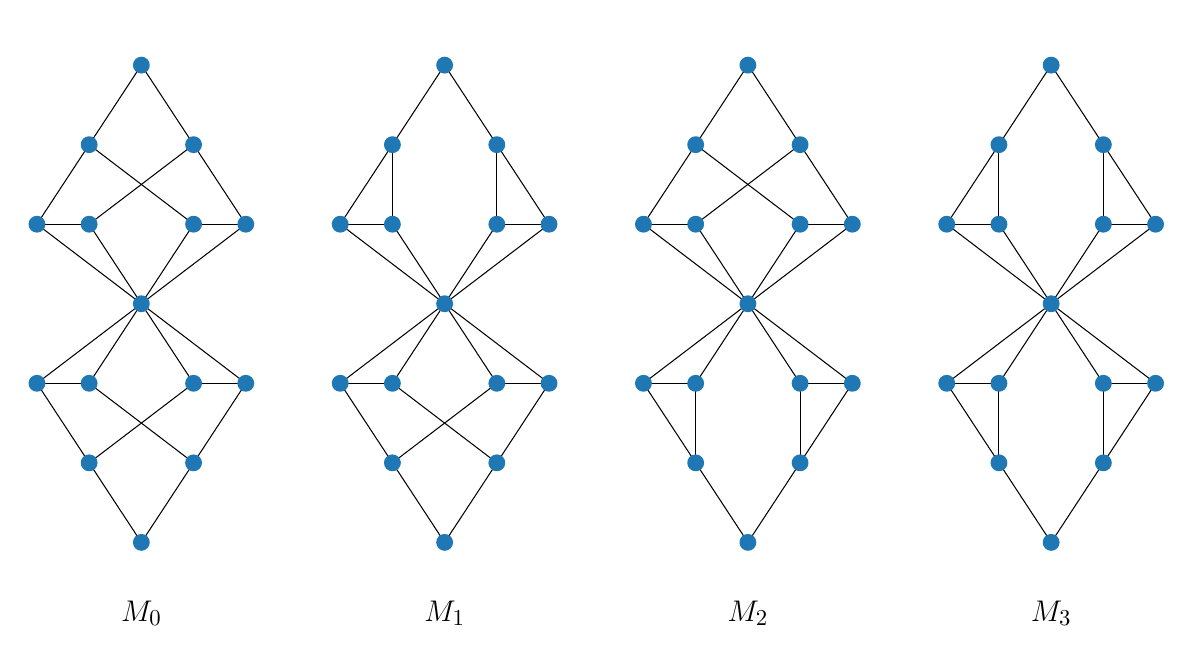}}
    \caption{\label{fig:m-plot}
    The graphs $M_0, M_1, M_2$ and $M_3$ \cite{naik2024iterative} that are not distinguishable by the 1-WL test.}
\end{figure}

\begin{table}[t]
	\centering
\tiny
    \begin{tabular}{lr|rrr|rrr|rrr|rr}
        \toprule
        Dataset & \#Graphs &  \multicolumn{3}{c}{\#Nodes} & \multicolumn{3}{c}{\#Edges} & \multicolumn{3}{c}{Diameter} & \#Node Labels & \#Classes \\
         &  & max & avg & min & max & avg & min & max & avg & min & & \\
        \midrule
        LongRings & 1\,200 & 100 & 100.0 & 100 & 100 & 100.0 & 100 & 50 & 50.0 & 50 & 5 & 3 \\
        EvenOddRings & 1\,200 & 16 & 16.0 & 16 & 16 & 16.0 & 16 & 8 & 8.0 & 8 & 16 & 4 \\
        EvenOddRingsCount & 1\,200 & 16 & 16.0 & 16 & 16 & 16.0 & 16 & 8 & 8.0 & 8 & 16 & 2 \\
        CSL \cite{DBLP:journals/combinatorica/CaiFI92} & 150 & 41 & 41.0 & 41 & 82 & 82.0 & 82 & 10 & 6.0 & 4 & 1 & 10 \\
        Snowflakes & 1\,000 & 180 & 112.5 & 45 & 300 & 187.5 & 75 & 18 & 15.5 & 13 & 2 & 4 \\

        \bottomrule
    \end{tabular}
    \caption{\label{tab:synthetic-datasets}
    Details of the synthetic datasets used in the experiments.}
\end{table}

\subsection{Results}\label{subsec:results}

\paragraph{Real-World Datasets} The results on the real-world datasets (\Cref{tab:real-world}) show that RuleGNNs are able to outperform the state-of-the-art graph classification algorithms
in the setting of~\cite{Errica2019AFC} even if we add all the additional label information that RuleGNNs use to the input features of the graph neural networks (see the (features) results in~\Cref{tab:real-world}).
This shows that the structural encoding of the additional label information is crucial for the performance of the graph neural networks and not replaceable by using additional input features.
Moreover, the results show that the Weisfeiler-Leman subtree kernel~\cite{DBLP:journals/jmlr/ShervashidzeSLMB11}
is the best performing graph classification algorithm on NC1, NCI109 and Mutagenicity.
For IMDB-BINARY and IMDB-MULTI our approach performs worse than the state-of-the-art graph classification algorithms that are not evaluated within the same experimental setup.
This might be the result of different splits of the datasets and the different evaluation setup or the fact that we have not found the best rule for these datasets.

\paragraph{Synthetic Datasets} The results on the synthetic benchmark graph dataset (\Cref{tab:synthetic}) show that the expressive power of RuleGNNs is higher than that of the standard message passing model.
Moreover, the integration of expert knowledge in the form of rules leads to a significant improvement in the performance of the model.
In fact, CLS and Snowflakes are not solvable by the message passing model because they are not distinguishable by the 1-WL test.
The results on LongRings show that long range dependencies can be easily captured by RuleGNNs and also dependencies between nodes of different distances as in case of the EvenOddRings dataset can be easily encoded by appropriate rules.

\begin{table}[t]
	\centering
\begin{adjustbox}{width=1\textwidth}
\tiny
	\begin{tabular}{lcccc|cc}
		\toprule
        & \textbf{NCI1} & \textbf{NCI109}  & \textbf{Mutagenicity} & \textbf{DHFR} & \textbf{IMDB-B}& \textbf{IMDB-M}\\
		\midrule
        Baseline (NoG) \cite{Schulz2019OnTN}&$69.2\pm 1.9$&$68.4\pm 2.2$&$74.8\pm 1.8$&$71.8\pm 5.3$&$71.9\pm 4.8$&$47.7\pm 4.0$\\
		WL-Kernel\cite{DBLP:journals/jmlr/ShervashidzeSLMB11}&\textcolor{red}{$\mathbf{85.2\pm 2.3}$}&\textcolor{red}{$\mathbf{85.0\pm 1.7}$}&\textcolor{red}{$\mathbf{83.8\pm 2.4}$}&$83.5\pm 5.1$&$71.8\pm 4.5$&$51.9\pm 5.6$\\
		\hline
        DGCNN\cite{DBLP:conf/aaai/ZhangCNC18}&$76.4\pm 1.7$&$73.0\pm 2.4$&$77.0\pm 2.0$&$72.6\pm 3.1$&$69.2\pm 3.0$&$45.6\pm 3.4$\\
		DGCNN (features)&$73.6\pm 1.0$& $72.5\pm 1.5$ &$76.3\pm 1.2$&$76.1\pm 3.4$&$69.1\pm 3.5$&$45.8\pm 2.9$\\
		GraphSage\cite{Hamilton2017InductiveRL}&$76.0\pm 1.8$&$77.1\pm 1.8$&$79.8\pm 1.1$&$80.7\pm 4.5$&$68.8\pm 4.5$&$47.6\pm 3.5$\\
		GraphSage (features)&$79.4\pm 2.2$&$78.6\pm 1.6$&$80.1\pm 1.3$&$82.4\pm 3.9$&$69.7\pm 3.1$&$46.6\pm 4.8$\\
        GIN\cite{DBLP:conf/iclr/XuHLJ19}&$80.0\pm 1.4$&$79.7\pm 2.0$&$81.9\pm 1.4$&$79.1\pm 4.4$&$71.2\pm 3.9$&$48.5\pm 3.3$\\
		GIN (features)& $77.3\pm 1.8$ & $77.7\pm 2.0$ & $80.6\pm 1.3$ & $81.8\pm 5.1$ & $70.9\pm 3.8$ & $48.3\pm 2.7$\\
		\midrule
		GSN (paper) \cite{DBLP:journals/pami/BouritsasFZB23} &$83.5\pm 2.3$&-&-&-&\textcolor{red}{$77.8\pm 3.3$}&\textcolor{red}{$54.3\pm 3.3$}\\
		CIN (paper) \cite{DBLP:conf/nips/BodnarFOWLMB21} &$83.6\pm 1.4$&$84.0\pm 1.6$&-&-&$75.6\pm 3.7$&$52.7\pm 3.1$\\
		SIN (paper)\cite{DBLP:conf/icml/BodnarF0OMLB21} & $82.7\pm 2.1$& - & - & - & $75.6\pm 3.2$ & $52.4\pm 2.9$\\
		PIN (paper) \cite{DBLP:conf/aaai/Truong024} & $85.1\pm 1.5$& $84.0\pm 1.5$ & - & - & $76.6\pm 2.9$ & -\\
		\hline
        \textbf{RuleGNN} & $82.8 \pm 2.0$ & $83.2 \pm 2.1$ & $81.5 \pm 1.3 $ & \textcolor{red}{$\mathbf{84.3 \pm 3.2}$} & $\mathbf{75.4 \pm 3.3}$ & $\mathbf{52.0 \pm 4.3}$\\
		\bottomrule
	\end{tabular}
    \end{adjustbox}
	\caption{\label{tab:real-world}
	Test set performance of several state-of-the-art graph classification algorithms averaged over three different runs and $10$ folds.
	The $\pm$ values report the standard deviation over the $10$ folds.
	The overall best results are colored red and the best ones obtained for the fair comparison from~\cite{Errica2019AFC} are in bold.
	The (features) variants of the algorithms use the same information as the RuleGNN as input features additionally to node labels.
	The (paper) results are taken from the respective papers using another experimental setup.
	}
\end{table}

\begin{table}[t]
	\centering
\tiny
	\begin{tabular}{lccc|cc}
		\toprule
		  & LongRings & EvenOddRings & EvenOddRingsCount & CSL & Snowflakes \\
		\midrule
		{Baseline (NoG)} \cite{Schulz2019OnTN}  & $30.17 \pm 3.2$ & $22.25 \pm 3.0$  & $47.9 \pm 3.9$  & $10.0 \pm 0.0$ & $27.3\pm 5.3$\\
		{WL-Kernel} \cite{DBLP:journals/jmlr/ShervashidzeSLMB11}  & $\mathbf{100.0 \pm 0.0}$ & $26.83 \pm 4.2$  & $47.8 \pm 4.3$ & $10.0 \pm 0.0$ & $27.9 \pm 4.1$\\
		\hline
		{DGCNN} \cite{DBLP:conf/aaai/ZhangCNC18} & $29.9\pm 2.6$ & $28.4\pm 2.5$ & $59.1\pm 5.2$ & $10.0\pm0.0$ & $26.0\pm 3.3$\\
		{GraphSAGE} \cite{Hamilton2017InductiveRL} & $29.8\pm 2.8$ & $24.9\pm 2.7$ & $51.3\pm 1.9$ & $10.0\pm0.0$ & $25.0 \pm 1.8$ \\
		{GIN} \cite{DBLP:conf/iclr/XuHLJ19} & $32.0\pm 3.1$ & $26.8\pm 2.5$ & $51.0 \pm 3.7$ & $10.0\pm0.0$ &$24.5 \pm 2.2$\\
		\textbf{RuleGNN}  & $99.0 \pm 3.3$ & $\mathbf{90.2 \pm 7.2}$ & $\mathbf{100.0 \pm 0.0}$ & $\mathbf{100.0 \pm 0.0}$&$\mathbf{97.9\pm 3.2}$\\
		\bottomrule
	\end{tabular}
	\caption{\label{tab:synthetic}
	Test set performance of several state-of-the-art graph classification algorithms averaged over three different runs and $10$ folds.
	The $\pm$ values report the standard deviation over the $10$ folds.
	The best results are highlighted in bold.
	}
\end{table}

\paragraph{Preprocessing and Trainin Details}
\Cref{tab:preprocessing-times} shows more details of training of RuleGNNs on the different datasets.
In particular, we see that except for the DHFR dataset we need less than $12$ epochs on average to reach the best result.
This shows that our approach is very efficient and converges quickly.
At the first glance the average time per epoch seems to be very high which has two reasons.
One is also mentioned in~\cite{DBLP:journals/pami/HanHSYWW22} that there is a gap between the theoretical and practical runtime of dynamic neural networks because
the implementation in PyTorch is not optimized for dynamic neural networks.
The other reason is that our computations run in parallel, i.e., we are able to run all the three runs and $10$ folds in parallel on the same machine which produces some overhead but is more efficient than running the experiments sequentially.
As stated above the preprocessing times (\Cref{tab:preprocessing-times}) are not relevant for the experiments as they are only needed once.
The third column shows the time needed to compute all the pairwise distances between the nodes of the graph.
The fourth column shows the time needed to compute the node labels used for the best model.
The most preprocessing time is needed for IMDB-BINARY and IMDB-MULTI because the graphs are much denser than the other datasets.
For the synthetic datasets except for CSL and Snowflakes we do not need any label preprocessing time as the original node labels are used.
\begin{table}
    \tiny
    \centering
    \begin{tabular}{lrrrrr}
        \toprule
        Dataset & Best Epoch & Avg. Epoch (s) & Preproc. Distances (s) & Preproc. Labels (s) & \#Graphs \\
        \midrule
        NCI1 &$8.3 \pm 5.3$&$377.1 \pm 20.7$& 2.0 & 11.9 & 4\,110 \\
        NCI109 &$6.4 \pm 2.9$&$386.7 \pm 1.9$& 2.4 & 13.2 & 4\,127 \\
        Mutagenicity &$10.1 \pm 4.1$&$575.8 \pm 66.4$& 2.2 & 15.2 & 4\,337 \\
        DHFR &$24.1 \pm 14.6$&$44.4 \pm 9.0$& 0.7 & 3.1& 756 \\
        IMDB-BINARY &$12.3 \pm 4.6$&$24.3 \pm 0.9$& 0.2 & 206.5 & 1\,000 \\
        IMDB-MULTI &$7.7 \pm 3.5$&$19.6 \pm 1.3$& 0.2 & 195.0& 1\,500 \\
        \midrule
        LongRings &$195.2 \pm 15.1$&$0.7 \pm 0.2$& 6.6 & - & 1\,200 \\
        EvenOddRings &$177.1 \pm 15.2$&$1.2 \pm 0.3$& 0.2 & - & 1\,200 \\
        EvenOddRingsCount &$200.0 \pm 0.0$&$0.5 \pm 0.1$& 0.1 & - & 1\,200 \\
        CSL &$50.0 \pm 0.0$&$1.6 \pm 0.0$& 0.1 & 11.8 & 150 \\
        Snowflakes &$192.7 \pm 18.9$&$0.5 \pm 0.1$& 7.1 & 116.8 & 1\,000 \\
        \bottomrule
    \end{tabular}
    \caption{\label{tab:preprocessing-times}
    Runtimes and preprocessing times of the different datasets used in the experiments.
    All values are averaged over the best runs.
    The first column shows the best epoch (highest validation accuracy), the second the average time per epoch, the third the time needed to compute all the pairwise distances between the nodes of the graph, the fourth the time needed to compute the node labels used for the best model and the last the number of graphs in the dataset.}
\end{table}

\paragraph{Architecture Details}

\Cref{tab:hyperparameters} provides an overview of the different architectures used in the experiments that achieved the best results on the validation set.
One advantage of our approach is that messages can be passed over long distances.
Hence, except for the EvenOddRings dataset we used only one layer and the output layer.
In case of NCI1, NCI109, Mutagenicity it turns out that the best model uses the Weisfeiler-Leman rule with $k=2$ iterations.
We restricted the number of maximum labels considered to $500$ which results in $250000$ learnable parameters for the weight matrix and $500$ for the bias vector.
For the output layer we used the bound of $50000$ learnable parameters which was larger than the number of different Weisfeiler-Leman labels in the second iteration.
Interestingly, for NCI1 and NCI109 the best validation accuracy was achieved if considering node pairs with distances from $1$ to $10$,
while in case of Mutagenicity the best model uses node pairs with distances from $1$ to $3$.
We also tested different small patterns, e.g., simple cycles, but they did not improve the results.
For DHFR the best model uses simple cycles with length at most $10$ as patterns for the output layer.
We also tested the Weisfeiler-Leman rule in this case but the validation accuracy was lower.
For IMDB-BINARY and IMDB-MULTI the best model uses the patterns simple cycles with length at most $10$, the triangle and a single edge.
Note that counting the embedding of a single edge as pattern is equivalent to the degree of the node.
We also tested the Weisfeiler-Leman rule but the validation accuracy was lower\footnote{See \href{https://github.com/fseiffarth/RuleGNNCode}{https://github.com/fseiffarth/RuleGNNCode} for a full list of tested hyperparameters.}.
As a next step it would be interesting to consider more rules, rules that come from expert knowledge or also deeper architectures with more rule based layers concatenated.
Regarding the number of learnable parameters we would like to mention that the number is relatively high but lots of parameters are not used in the weight matrix.
Hence, it might be possible to prune the set of learnable parameters by removing those that are not used or those that have a small absolute value.

For the synthetic datasets we use ``expert knowledge'' to define the rules.
Hence we did not tested other rules than those in~\Cref{tab:hyperparameters}.
For LongRings, EvenOddRings and EvenOddRingsCount we used the original node labels for the rule based layers.
In case of EvenOddRings we used two layers.
The first layer considers only node pairs with distance $8$ and collects all the necessary information of opposite nodes.
The second layer that considers only node pairs with distance $4$ and collects the information of the nodes that are $4$ hops away from the nodes with label $0$, see also~\Cref{fig:interpretability}.
For CSL we used as patterns all simple cycles with length at most $10$.
For the Snowflakes dataset we used the patterns, cycle of length $4$ and $5$ and collect the information of all nodes that have pairwise distance $3$.
In this way the RuleGNN is able to distinguish the graphs $M_0, M_1, M_2$ and $M_3$ that are not distinguishable by the 1-WL test.
For the output layer we used the Weisfeiler-Leman rule with $k=2$ iterations to collect the relevant information.

\begin{table}
    \tiny
    \centering
    \begin{tabular}{lrrrrr}
        \toprule
        Dataset & Rules & \multicolumn{3}{c}{Hyperparameters} & \#Learnable Parameters  \\
        & & $k$ & $\mathcal{D}$ & $L$ &  per Layer \\
        \midrule
        NCI1 & wl & 2 & \{1,\ldots,10\}  & 500 & 2\,500\,500 \\
             & wl & 2 & - & 50000 & 4\,220\\\\
        NCI109 & wl & 2 & \{1,\ldots,10\}  & 500 & 2\,500\,500 \\
             & wl & 2 & - & 50000 & 4\,336\\\\
        Mutagenicity & wl & 2 & \{1,\ldots,3\}  & 500 & 750\,500\\
             & wl & 2 & - & 50000 & 4\,972\\\\
        DHFR & wl & 2 & \{1,\ldots,6\} & 500 & 1\,382\,880 \\
             & pattern: (simple\_cycles$\leq 10$) & - & - & - & 112 \\\\
        IMDB-BINARY & pattern: (triangle, edge)& - & \{1,2\} & - & 963\,966\\
             & pattern: (induced\_cycles$\leq 5$) & - & - & - & 990 \\\\
        IMDB-MULTI & pattern: (triangle, edge) & - & \{1,2\} & - & 551\,775 \\
             & pattern: (triangle, edge) & 10 & - & - & 1\,578\\\\
        \midrule
        LongRings & labels & - & \{25\} & -& 30 \\
                    & labels & - & - & - &18\\\\
        EvenOddRings & labels & - & \{8\} & - &272 \\
                    & labels & - & \{4\} & - &272\\
                    & labels & - & - & - & 68\\\\
        EvenOddRingsCount & labels & - & \{8\} & - & 272 \\
                    & labels & - & - & - & 34\\\\
        CSL & pattern: (simple\_cycles$\leq 10$) &  - & \{1\} & -& 8930\\
            & pattern: (simple\_cycles$\leq 10$) & - & - & -& 950\\\\
        Snowflakes & pattern: (cycle\_4, cycle\_5) & - & \{3\} & - &90 \\
                    & wl & 2 & -  & -&20 \\\\
    \end{tabular}
    \caption{\label{tab:hyperparameters}
    Best architectures per dataset. The column \textit{Rule} shows the type of rule used in the model, \textit{wl} stands for the Weisfeiler-Leman labeling, \textit{pattern} for the pattern based labeling and \textit{labels} for the original node labels.
    The last layer is always an aggregation layer. While the others are Weisfeiler-Leman layers or Pattern Counting layers based on the labeling of the nodes.
    The column \textit{Hyperparameters} shows the hyperparameters used in the model, $k$ is the number of iterations of the Weisfeiler-Leman rule, $\mathcal{D}$ is the set of valid pairwise distances considered and $L$ is the bound for the number of different node labels considered.
    The column \textit{\#Learnable Parameters} shows the number of learnable parameters in the model.}
\end{table}

\paragraph{Interpretability of RuleGNNs}
Each learnable parameter of RuleGNNs used for the weight matrices can be interpreted in terms of the importance of a connection between two nodes
in a graph with respect to their labels and their shared property (in our case the distance).
That is, each model provides the relevance of two nodes $i, j$ in a graph with labels $l(i), l(j)$ and distance $d(i, j)$.
In Figures~\ref{fig:weight-evaluation} and~\ref{fig:interpretability} we see how the network has learned the importance of different connections between nodes
for different distances and labels.
The weights are visualized by arrows (thickness corresponds to the absolute value and the color to the sign).
The biases are visualized by the nodes (size corresponds to the absolute value and the color to the sign).
\Cref{fig:weight-evaluation} shows an example of the relevance of the weights for graphs from the DHFR and IMDB-BINARY datasets using the best model.
We can see that in case of DHFR the RuleGNN has learned to pass messages from the outer nodes to ring nodes.
Some ring nodes seem to be more important than others.
It is an interesting open question if these connections can be interpreted in a chemical context.
In case of the IMDB-BINARY dataset we can see that the RuleGNN has learned to pass messages to some specific nodes.
It would be interesting to further investigate if these nodes have a specific meaning in the context of the dataset.
\Cref{fig:interpretability} shows an example of the learned parameters for our synthetic datasets.
Considering the dataset RingEvenOdd in~\Cref{fig:evenoddrings} we see that in the first layer the RuleGNN passes the messages between opposite nodes as given by the rule.
In the second layer it has learned the relevant information, i.e., to collect the information from the nodes that have distance $4$ to the node with label $0$ (dark blue node).
All other connections of distance $4$ have a smaller weight, i.e., are less important.
For the Snowflakes dataset~\Cref{fig:snowflakes} we see that the RuleGNN has learned to distinguish between the four different subgraphs $M_0, M_1, M_2$ and $M_3$ glued to the central circle, showing that the power of the RuleGNN goes beyond the $1$-WL test.
\begin{figure}[htbp]
  \centering
  	\begin{subfigure}{0.28\textwidth}\centering
  		{\includegraphics[width=\textwidth]{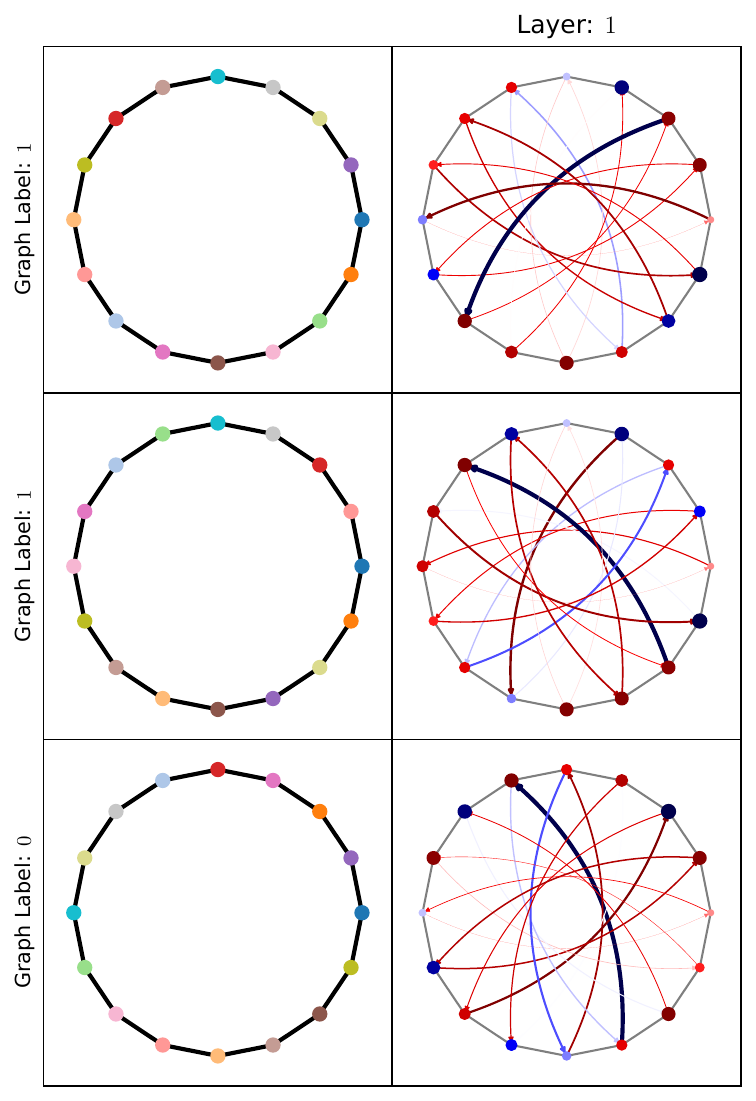}}
  		\caption{EvenOddCount\label{fig:longrings}}
		  	\end{subfigure}
  	\begin{subfigure}{0.41\textwidth}\centering
		  		{\includegraphics[width=\textwidth]{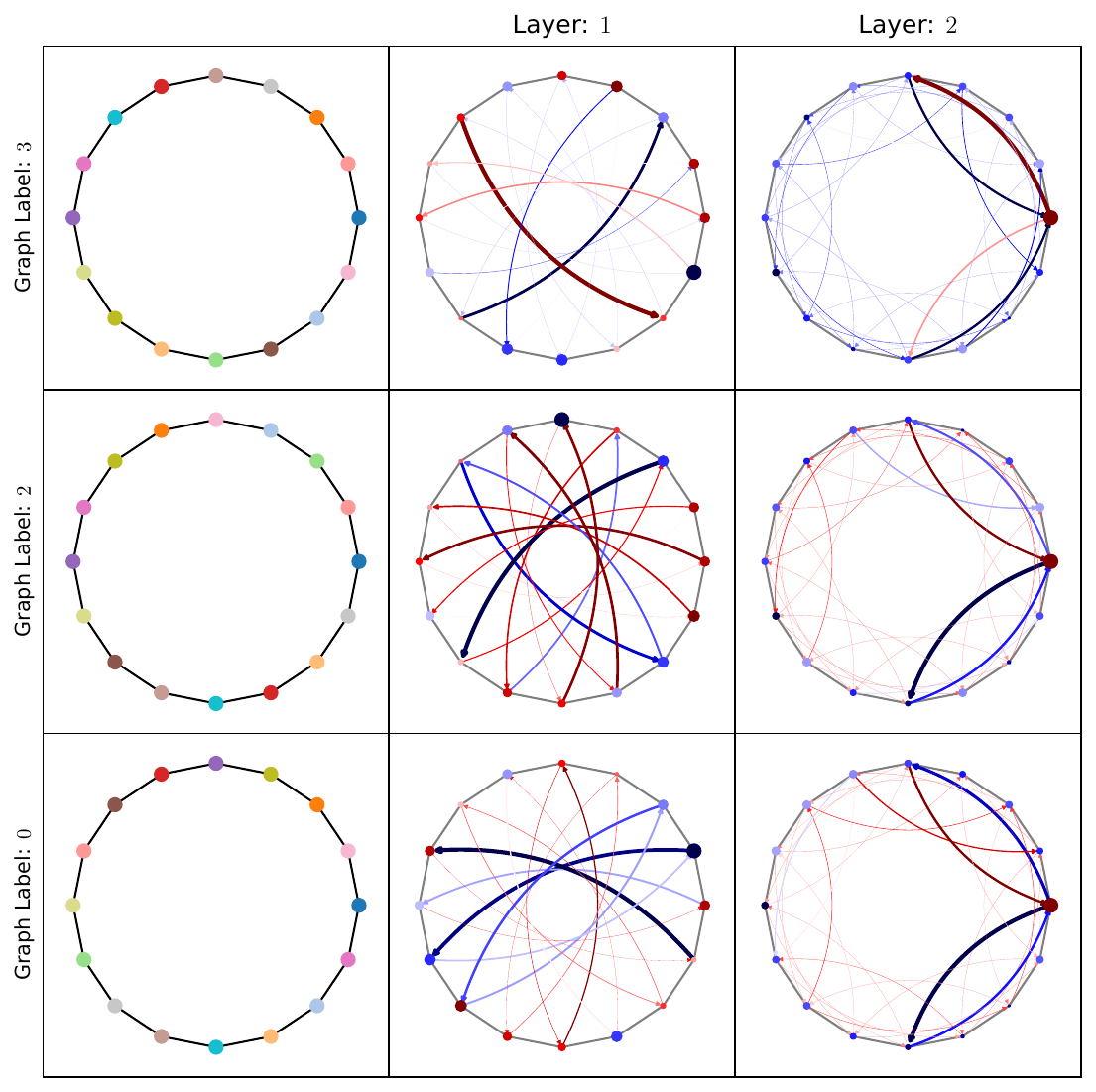}}
  		\caption{EvenOddRings\label{fig:evenoddrings}}
\end{subfigure}
  	\begin{subfigure}{0.28\textwidth}\centering
	{\includegraphics[width=\textwidth]{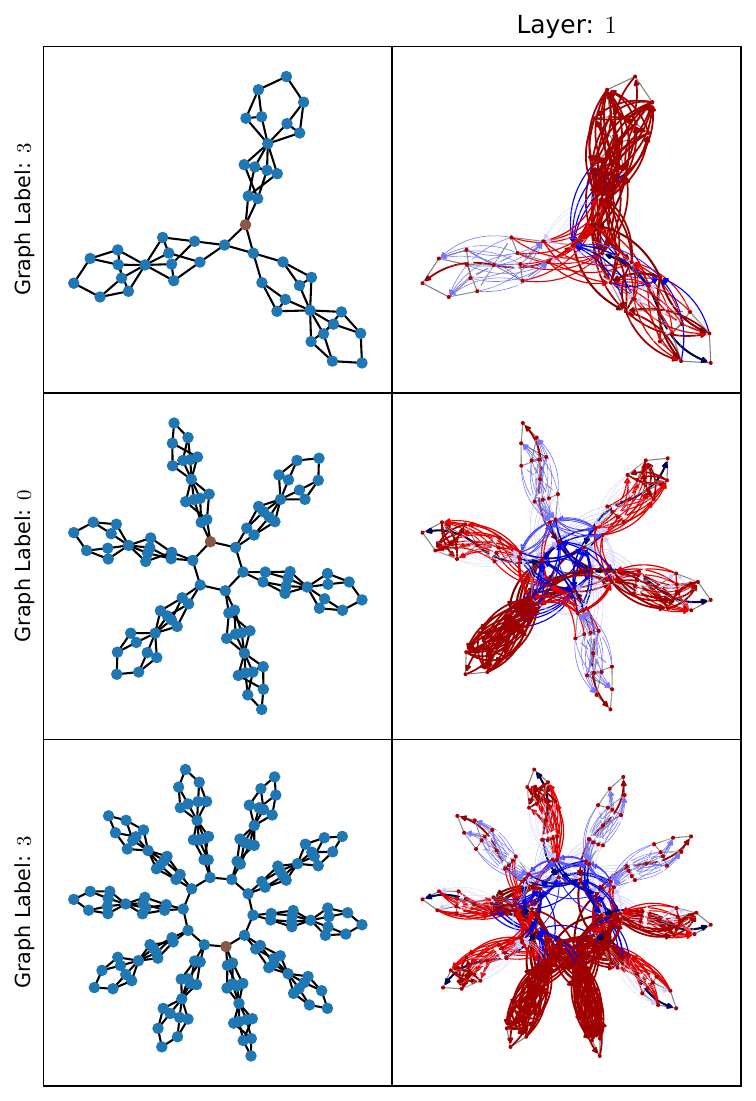}}
  		\caption{Snowflakes\label{fig:snowflakes}}
\end{subfigure}
  \caption{\label{fig:interpretability}Visualization of the learned parameters for the EvenOddRingsCount (\subref{fig:longrings}), EvenOddRings (\subref{fig:evenoddrings}) and Snowflakes (\subref{fig:snowflakes}) dataset.
  The first column shows the graphs and the colors of the nodes represent the different node labels.
  The other columns show the learned weights and biases for the respective rule based layer.
  The message passing weights are visualized by arrows (thicker for higher absolute values)
	  and the biases are visualized by the size of the node (red for positive and blue for negative weights).}
\end{figure}

\section{Related Work}\label{sec:RelatedWork}
Dynamic neural networks have been proven to be more efficient~\cite{DBLP:conf/iclr/HuangCLWMW18}, have more representation power~\cite{DBLP:conf/nips/YangBLN19} and better interpretability~\cite{DBLP:conf/nips/WangLHSYH20} compared to static neural networks, see~\cite{DBLP:journals/pami/HanHSYWW22} for a survey on this topic.
Variants of dynamic neural networks are successfully applied to different tasks on image data~\cite{DBLP:conf/nips/WangLHSYH20,DBLP:conf/cvpr/YangCL16,DBLP:conf/cvpr/ZhuHLD19}, natural language processing~\cite{DBLP:conf/eacl/XuM23} and graph classification~\cite{DBLP:conf/cvpr/SimonovskyK17}.
Using the categories of dynamic neural networks proposed in~\cite{DBLP:journals/pami/HanHSYWW22} our approach can be seen as a \textit{dynamic routing} network which is a subcategory of
sample dependent dynamic neural network which are networks that are dependent on the input data.
Examples of dynamic routing networks are given by~\cite{DBLP:conf/icml/McGillP17,DBLP:conf/nips/SabourFH17,DBLP:conf/eccv/WangYDDG18}.
In particular, our dynamic architecture adaptation is similar to neural architecture search~\cite{DBLP:conf/iclr/ZophL17,DBLP:conf/icml/YouLHX20} (NAS) where the goal is to find the best architecture for a given task.
We assume that the ``best'' architecture is given by the ``best'' rules.
Thus, in our setting neural architecture search translates into the search for the best rules.
The architecture of rule based neural networks depends on the input data and hence our approach can be seen as a neural architecture search that tries to learn the best architecture per input sample~\cite{DBLP:conf/aaai/ChengLJ0S20},
guided by the predefined rules.

Regarding the specific application to graphs we would like to note that graph neural networks based on the message passing paradigm~\cite{DBLP:conf/icml/GilmerSRVD17}
have been successfully applied to the task of graph classification~\cite{Hamilton2017InductiveRL,DBLP:conf/iclr/KipfW17,Velickovic2017GraphAN,DBLP:conf/iclr/XuHLJ19}.
Some limitations that are addressed by our approach have been considered in the literature.
For example $k$-hop approaches that aggregate information over long distances have been considered~\cite{DBLP:conf/icml/Abu-El-HaijaPKA19,DBLP:journals/nn/NikolentzosDV20,DBLP:conf/ijcai/WangY0L21}.
To overcome the limitations of $1$-WL test recent algorithms use additional information like subgraph structures or topological information to improve the performance~\cite{DBLP:conf/nips/BodnarFOWLMB21,DBLP:conf/icml/BodnarF0OMLB21,DBLP:journals/pami/BouritsasFZB23,DBLP:conf/aaai/Truong024}
In~\cite{Zhou2017GraphCA} the authors show that graph neural networks can learn chemical rules like the ortho-para rule for molecules which goes in the direction of interpretability.
To increase the interpretability and explainability of graph neural networks there exist different approaches
that provide insights into the prediction of the model~\cite{DBLP:conf/nips/YingBYZL19,DBLP:journals/corr/abs-2404-12356}.
Moreover, also dynamic approaches for graphs are considered in the literature~\cite{DBLP:conf/cvpr/SimonovskyK17}.
In contrast to these approaches and algorithms, we provide a simple and general scheme to overcome different limitations of graph neural networks at once.
While additional information used in other algorithms is mostly hard-coded, we are able to integrate expert knowledge by arbitrary rules.
In fact, we are not aware of any other graph neural network that is able to dynamically adjust the architecture based on the input graphs using predefined rules.

\section{Concluding Remarks}\label{sec:Conclusion}

Finally, we would like to discuss some limitations of our approach together with possible solutions.
Moreover, we present some concluding remarks including an outlook on future research directions.
In this work, we have only considered 1-dimensional input signals and node labels, i.e., our experimental results are restricted to graphs that do not have multidimensional node features.
Additionally, we have not considered edge features in our rules.
In principle, multidimensional node features and edge labels can be handled by our approach with the cost of increasing complexity.
For each graph we need to precompute the pairwise distances and store the positions of the weights in the weight-matrix.
This is a disadvantage regarding large and dense graphs as we need to store a large number of positions.
For dense graphs the number of positions can be quadratic in the number of nodes.
Thus, we need to find a trade-off between the performance of our model and the sparsity of the weight matrices for large and dense graphs.
To define a meaningful rule for a layer the input and output features need to be logically connected.
Fortunately, this is the case for graphs but for other domains this might be a limitation.
If it is not possible to define a formal rule using expert knowledge or additional information the number of possible rules that have to be tested can be very large.
Thus, it is an interesting question if it is possible to automatically learn a rule that fits the data or captures the expert knowledge in the best way.
As stated in~\cite{DBLP:journals/pami/HanHSYWW22} there is a ``gap between theoretical \& practical efficiency'' regarding dynamic neural networks, i.e., common libraries such as PyTorch or TensorFlow are not optimized for these approaches.

Summarizing our contributions, we have introduced a new type of neural network layer that dynamically arranges the learnable parameters in the
weight matrices and bias vectors according to formal rules.
On the one hand our approach generalizes classical neural network components such as fully connected layers and convolutional layers.
On the other hand we are able to apply rule based layers to the task of graph classification showing that expert knowledge can be integrated into the learning process.
Moreover, we have shown that for graph classification our approach gives rise to a more interpretable neural network architecture as every learnable parameter is related to a specific connection between input and output features.
This leads to some interesting open questions that we would like to address in future work.
For example, it is an interesting question if it is possible to automatically learn the best rule for a given task during training.
The experiments on graphs have shown that the number of relevant parameters, i.e., that are updated during training, is small compared to the total number of parameters.
Thus, for future work it would be interesting to investigate if it is possible to prune the number of parameters, i.e., to set parameters to zero, to increase the interpretability of the model, reduce the computational costs and the size of the model.
The visualization of the learned parameters suggest that the model has learned some abstract rules, e.g., regarding the DHFR dataset.
In fact, it is an open question if these rules have a theoretical background in chemistry.
If chemical rules are learned by our model, one could think of transferring the rules learned on one dataset to another dataset which could increase the predictive performance especially for small datasets.
Another advantage of our rule based approach is that it is easily integrable into existing architectures.
Hence, it would be interesting to consider other tasks like node classification or even other domains like images or text.

\bibliography{bibliography}
\bibliographystyle{plain}

\end{document}